\documentclass[letterpaper,10pt]{IEEEtran}

\usepackage{amsmath,amssymb,amsthm}
\usepackage{graphicx,url}
\usepackage[font=footnotesize]{subfig}
\usepackage{xspace}
\usepackage{fixltx2e}
\usepackage{caption}

\newtheorem{theorem}{Theorem}[section]
\newtheorem{proposition}[theorem]{Proposition}
\newtheorem{corollary}[theorem]{Corollary}
\newtheorem{definition}[theorem]{Definition}
\newtheorem{lemma}[theorem]{Lemma}
\newtheorem{remark}[theorem]{Remark}

\newtheorem{problem}[theorem]{Problem}
\newtheorem{assumption}[theorem]{Assumption}

\newcommand{\BB}{\mathcal{B}}

\newcommand{\RR}{\mathcal{R}}
\newcommand{\NN}{\mathcal{N}}

\newcommand{\nat}{\mathbb{N}}

\newcommand{\real}{{\mathbb{R}}}

\newcommand{\env}{\mathcal{E}}
\newcommand{\p}{\mathbf{p}}
\newcommand{\q}{\mathbf{q}}

\newcommand{\setdef}[2]{\{#1 \; | \; #2\}}

\newcommand{\vmax}{v_{\max}}

\newcommand{\vmin}{v_{\min}}
\newcommand{\num}{\operatorname{num}}
\newcommand{\den}{\operatorname{den}}

\newcommand{\unit}[1]{\ensuremath{~\mathrm{#1}}}

\newcommand\oprocendsymbol{\hbox{$\square$}}
\newcommand\oprocend{\relax\ifmmode\else\unskip\hfill\fi\oprocendsymbol}

\urldef{\smith}\url{slsmith@mit.edu}
\urldef{\rus}\url{rus@csail.mit.edu}
\urldef{\schwager}\url{schwager@mit.edu}

\title{\huge Persistent Robotic Tasks: \\ Monitoring and Sweeping in
  Changing Environments}

\author{Stephen L. Smith \qquad Mac Schwager \qquad Daniela Rus
\thanks{This material is based upon work supported in part by ONR-MURI
  Award N00014-09-1-1051} \thanks{The authors are with the Computer
  Science and Artificial Intelligence Laboratory, Massachusetts
  Institute of Technology, Cambridge MA 02139 (\smith, \schwager,
  \rus).  M. Schwager is also affiliated with the GRASP Laboratory,
  University of Pennsylvania, Philadelphia, PA 19104.}}

\begin{document}
\maketitle
\begin{abstract}

  We present controllers that enable mobile robots to persistently
  monitor or sweep a changing environment. The changing environment is
  modeled as a field which grows in locations that are not within
  range of a robot, and decreases in locations that are within range
  of a robot. We assume that the robots travel on given closed paths.
  The speed of each robot along its path is controlled to prevent the
  field from growing unbounded at any location. We consider the space
  of speed controllers that can be parametrized by a finite set of
  basis functions. For a single robot, we develop a linear program
  that is guaranteed to compute a speed controller in this space to
  keep the field bounded, if such a controller exists. Another linear
  program is then derived whose solution is the speed controller that
  minimizes the maximum field value over the environment. We extend
  our linear program formulation to develop a multi-robot controller
  that keeps the field bounded. The multi-robot controller has the
  unique feature that it does not require communication among the
  robots. Simulation studies demonstrate the robustness of the
  controllers to modeling errors, and to stochasticity in the
  environment.

\end{abstract}

\section{Introduction}
\label{Sec:Introduction}
In this paper we treat the problem of controlling robots to
perpetually act in a changing environment, for example to clean an
environment where material is constantly collecting, or to monitor an
environment where uncertainty is continually growing.  Each robot has
only a small footprint over which to act (e.g. to sweep or to sense).
The difficulty is in controlling the robots to move so that their
footprints visit all points in the environment regularly, spending
more time in those locations where the environment changes quickly,
without neglecting the locations where it changes more slowly.  This
scenario is distinct from most other sweeping and monitoring scenarios
in the literature because the task cannot be ``completed.''  That is
to say, the robots must continually move to satisfy the objective.  We
consider the situation in which robots are constrained to move on
fixed paths, along which we must control their speed.  We consider
both the single robot and multi-robot cases.
Figure~\ref{fig:sample_PM} shows three robots monitoring an
environment using controllers designed with our method.
\begin{figure}
\centering
  \includegraphics[width=0.8\linewidth]{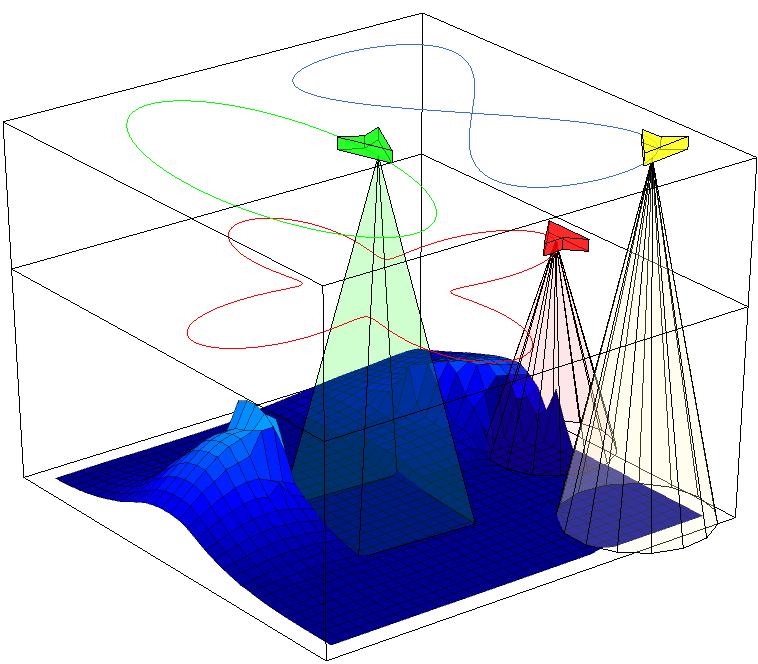}
  \caption{A persistent monitoring task using three robots with
    heterogeneous, limited range sensors.  The surface shows the
    accumulation function, indicating the quantity of material to be
    removed in a cleaning application, or the uncertainty at each
    point in a sensing application.  The accumulation function grows
    when a robot's footprint is not over it, and decreases when the
    footprint is over it.  Each robot controls its speed along its
    prescribed path so as to keep the surface as low as possible
    everywhere.}
  \label{fig:sample_PM}
\end{figure}

We model the changing environment with a scalar valued function
defined, which we call the \emph{accumulation function}. The function
behaves analogously to dust accumulating over a floor.  When a robot's
footprint is not over a point in the environment, the accumulation
function grows at that point at a constant rate, as if the point were
collecting dust.  When a robot's footprint is over the point, the
accumulation function decreases at a constant rate, as if the dust
were being vacuumed by the robot.  The rates of growth and decrease
can be different at different points in the environment.

This model is applicable to situations in which the task of the robot
is to remove material that is collecting in the environment, for
example cleaning up oil in the ocean around a leaking
well~\cite{KakalisVentikosJouranlOfHazardousMaterials08}, vacuuming
dirt from the floor of a
building~\cite{MacKenzieBalchAAAI93RobotVacuuming}, or tending to
produce in an agricultural field~\cite{NC-NA.ea:09}.  It is also
applicable to monitoring scenarios in which the state of the
environment changes at each point, and we want to maintain up-to-date
knowledge of this changing state.  Some examples include monitoring
temperature, salinity, or chlorophyll in the
ocean~\cite{SmithJFR10UnderwaterGliders}, maintaining estimates of
coral reef health~\cite{DunbabinAusConfOnRobAndAuto04ReefRobot}, or
monitoring traffic congestion over a
city~\cite{SrinivasanVSSN04AirborneTrafficSurveillance}.  These
applications all share the property that they can never be completed,
because the environment is always changing.  If the robot were to stop
moving, the oil would collect beyond acceptable levels, or the
knowledge of the ocean temperature would become unacceptably outdated.
For this reason we call these applications \emph{persistent tasks}.
In this paper, we assume that the model is known either from the
physics of the environment, from a human expert, or from an initial
survey of the environment.  However we show analytically and in
simulation that our controllers are robust to model errors.

We focus on the situation in which the robots are given fixed, closed
paths on which to travel, and we have to carry out the persistent task
only by regulating their speed.  This is relevant in many real-world
scenarios.  For example, in the case that the robots are autonomous
cars, they may be constrained to move on a particular circuit of roads
in a city to avoid traffic congestion, or in the case of autonomous
aircraft, they may be constrained to fly along a particular path to
stay away from commercial air traffic, or to avoid being detected by
an adversary.  Even in the case of vacuuming the floor of a building,
the robots may be required to stay along a prescribed path so as not
to interfere with foot traffic.  Furthermore, if we have the freedom
to plan the robot's path, we can employ an off-line planner to
generate paths that are optimal according to some metric (using the
planner in~\cite{SLS-DR:09a} for example), then apply the procedure in
this paper to control the speed along the path.  Decoupling the path
planning from the speed control in this way is a well-established
technique for dealing with complex trajectory planning
problems~\cite{KantZuckerIJRR86PathVelocityDecomposition}.

\subsection{Contributions}
\label{Sec:Contributions}

Our approach to the problem is to represent the space of all possible
speed controllers with a finite set of basis functions, where each
possible speed controller is a linear combination of those basis
functions.  A rich class of controllers can be represented in this
way.  Using this representation as our foundation, the main
contributions of this paper are the following.
\begin{enumerate}
\item We formally introduce the idea of persistent tasks for robots in
  changing environments, and propose a tractable model for designing
  robot controllers.

\item Considering the space of speed controllers parametrized by a
  finite set of basis functions, we formulate a linear program (LP)
  whose solution is a speed controller which guarantees that the
  accumulation function eventually falls below, and remains below, a
  known bound everywhere.  If the LP is infeasible, then no controller
  in the space that will keep the accumulation function bounded.

\item We formulate an LP whose solution is the
  optimal speed controller---that which eventually minimizes the
  maximum of the accumulation function over all locations.  

\item We generalize to the multi-robot case.  We find an LP whose
  solution is a controller which guarantees that the accumulation
  function falls below, and stays below, a known bound.  The
  multi-robot controller does not require communication between
  robots.

\end{enumerate} 
We do not find the optimal controller for the multi-robot case,
however, as it appears that this controller cannot be found as the
solution to an LP.  We demonstrate the performance of the controllers
in numerical simulations, and show that they are robust to stochastic
and deterministic errors in the environment model, and to unmodeled
robot vehicle dynamics.

It is desirable to cast our speed control problem as an LP since LPs
can be solved efficiently with off-the-shelf solvers~\cite{DGL:84}.
This is enabled by our basis function representation. The use of basis
functions is a common, and powerful method for function
approximation~\cite{EWC:00}, and is frequently used in areas such as
compressive sampling~\cite{EJC-MBW:08}, adaptive
control~\cite{RS-JJES:92}, and machine learning~\cite{TP-SS:03}. Our
LP formulations also incorporate both maximum and minimum limits on
the robot's speed, which can be different at different points on the
path. This is important because we may want a robot not to exceed a
certain speed around a sharp turn, for example, while it can go much
faster on a long straightaway. 
  
\subsection{Related Work}
\label{Sec:RelatedWork}
Our work is related to the large body of existing research on
environmental monitoring, sensor sweep coverage, lawn mowing and
milling, and patrolling.  In the environmental monitoring literature
(also called objective analysis in meteorological research
\cite{Gandin63ObjectiveAnalysis} and Kriging in geological research
\cite{CressieMathematicalGeology90Kriging}), authors often use a
probabilistic model of the environment, and estimate the state of that
model using a Kalman-like filter.  Then robots are controlled so as to
maximize a metric on the quality of the state estimate.  For example,
the work in~\cite{GrocholskyPhDThesis02} controls robots to move in
the direction of the gradient of mutual information.  More recently,
the authors of \cite{Lynch08} and~\cite{Yang08} control vehicles to
decrease the variance of their estimate of the state of the
environment.  This is accomplished in a distributed way by using
average consensus estimators to propagate information among the
robots.  Similarly, \cite{CortesTAC09KrigedKalmanFilter} proposes a
gradient based controller to decrease the variance of the estimate
error in a time changing environment.  In~\cite{Zhang10} sensing
robots are coordinated to move in a formation along level curves of
the environmental field.  In~\cite{GrahamCortesTAC10VoronoiTrajectory}
the authors find optimal trajectories over a finite horizon if the
environmental model satisfies a certain spatial separation property.
Also, in \cite{LeNyCDC09PathOptGP} the authors solve a dynamic program
(DP) over a finite horizon to find the trajectory of a robot to
minimize the variance of the estimate of the environment, and a
similar DP approach was employed in~\cite{BourgaultIROS02} over short
time horizons.  Many other works exist in this vein.

Although these works are well-motivated by the uncontested successes
of Kalman filtering and Kriging in real-world estimation
applications, they suffer from the fact that planning optimal
trajectories under these models requires the solution of an
intractable dynamic program, even for a static environment.  One must
resort to myopic methods, such as gradient descent (as
in~\cite{GrocholskyPhDThesis02,Lynch08,Yang08,Zhang10,CortesTAC09KrigedKalmanFilter}),
or solve the DP approximately over a finite time horizon (as
in~\cite{GrahamCortesTAC10VoronoiTrajectory,LeNyCDC09PathOptGP,BourgaultIROS02}).
Although these methods have great appeal from an estimation point of
view, little can be proved about the comparative performance of the
control strategies employed in these works.  The approach we take in
this paper circumvents the question of estimation by formulating a new
model of growing uncertainty in the environment.  Under this model, we
can solve the speed planning problem over \emph{infinite time}, while
maintaining \emph{guaranteed} levels of uncertainty in a
\emph{time-changing} environment.  Thus we have used a less
sophisticated environment model in order to obtain stronger results on
the control strategy.  Because our model is based on the analogy of
dust collecting in an environment, we also solve infinite horizon
sweeping problems with the same method.

Our problem in this paper is also related to sweep coverage, or lawn
mowing and milling problems, in which robots with finite sensor
footprints move over an environment so that every point in the
environment is visited at least once by a robot. Lawn mowing and
milling has been treated in~\cite{EMA-SPF-JSBM:00} and other works.
Sweep coverage has recently been studied in~\cite{IR-VLS-APN-HC:04},
and in \cite{KoenigSzymanskiLiuAnnalsofMathAI01Ants} efficient sweep
coverage algorithms are proposed for ant robots. A survey of sweep
coverage is given in~\cite{HC:01}. Our problem is significantly
different from these because our environment is dynamic, thereby
requiring continual re-milling or re-sweeping. A different notion of
persistent surveillance has been considered in~\cite{BB-JPH-JV:08}
and~\cite{BB-JR-JPH-MAV-JV:10}, where a persistent task is defined as
one whose completion takes much longer than the life of a robot. While
the terminology is similar, our problem is more concerned with the
task (sweeping or monitoring) itself than with power requirements of
individual robots.

A problem more closely related to ours is that of
patrolling~\cite{YC:04,YE-NA-GAK:07}, where an environment must be
continually surveyed by a group of robots such that each point is
visited with equal frequency. Similarly, in~\cite{NN-IK:08} vehicles
must repeatedly visit the cells of a gridded environment. Also,
continual perimeter patrolling is addressed in~\cite{DBK-RWB-RSH:08}.
In another related work, a region is persistently covered
in~\cite{PFH-DS-MWS:07} by controlling robots to move at constant
speed along predefined paths. Our work is different from these,
however, in that we treat the situation in which different parts of
the environment may require different levels of attention. This is a
significant difference as it induces a difficult resource trade-off
problem as one typically finds in queuing theory~\cite{LK:75}, or
dynamic vehicle routing~\cite{DJS-GJvR:93a,SLS-MP-FB-EF:09a}.
In~\cite{MA-PS:06}, the authors consider unequal frequency of visits
in a gridded environment, but they control the robots using a greedy
method that does not have performance guarantees.

Indeed, our problem can be seen as a dynamic vehicle routing problem
with some unique features. Most importantly, we control the speed of
the robots along a pre-planned path, whereas the typical dynamic
vehicle routing problem considers planning a path for vehicles that
move at constant speed.  Also, in our case all the points under a
robot's footprint are serviced simultaneously, whereas typically
robots service one point at a time in dynamic vehicle routing.
Finally, in our case servicing a point takes an amount of time
proportional to the size of the accumulation function at that point,
whereas the service time of points in dynamic vehicle routing is
typically independent of that points wait-time.

The paper is organized as follows. In
Section~\ref{sec:problem_formulation} we set up the problem and define
some basic notions. In Section~\ref{sec:stability_and_optimality} two
LPs are formulated, one of which gives a stabilizing controller, and
the other one an optimal controller. Multiple robots are addressed in
Section~\ref{sec:multi_robot}. The performance and robustness of the
controllers are illustrated in simulations in
Section~\ref{sec:simulations}. Finally, Section~\ref{sec:conclusions}
gives conclusions and extensions.

\section{Problem Formulation and Stability Notions}
\label{sec:problem_formulation}

In this section we formalize persistent tasks, introduce the notion of
a field stabilizing controller, and provide a necessary and sufficient
conditions for field stability.

\subsection{Persistent Tasks}

Consider a compact environment $\env \subset \real^2$, and a finite
set of points of interest $Q\subseteq\env$.  The environment contains
a closed curve $\gamma : [0,1] \to \real^2$, where $\gamma(0) =
\gamma(1)$.  (See Figure~\ref{fig:curve_setup} for an illustration.)
The curve is parametrized by $\theta\in[0,1]$, and we assume without
loss of generality that $\theta$ is the arc-length parametrization.
The environment also contains a single robot (we will generalize to
multiple robots in Section~\ref{sec:multi_robot}) whose motion is
constrained along the path $\gamma$.  The robot's position at a time
$t$ can be described by $\theta(t)$, its position along the curve
$\gamma$.  The robot is equipped with a sensor with finite footprint
$\BB(\theta)\subset \env$ (For example, the footprint could be a disk
of radius $r$ centered at the robot's position).\footnote{For a ground
  robot or surface vessel, the footprint may be the robot's field of
  view, or its cleaning surface.  For a UAV flying at constant
  altitude over a 2D environment, the footprint could give the portion
  of the environment viewable by the robot.  }  Our objective is to
control the speed $v$ of the robot along the curve.  We assume that
for each point $\theta$ on the curve, the maximum possible robot speed
is $\vmax(\theta)$ and the minimum robot speed is $\vmin(\theta)>0$.
This allows us to express constraints on the robot speed at different
points on the curve.  For example, for safety considerations, the
robot may be required to move more slowly in certain areas of the
environment, or on high curved sections of the path.  To summarize,
the robot is described by the triple $\RR :=(\BB,\vmin,\vmax)$.

A time-varying field $Z:Q\times \real_{\geq 0} \to \real_{\geq 0}$,
which we call the \emph{accumulation function}, is defined on the
points of interest $Q$.  This field may describe a physical quantity,
such as the amount of oil on the surface of a body of water.
Alternatively, the field may describe the robot's uncertainty about
the state of each point of interest.  We assume that at each point
$\q\in Q$, the field $Z$ increases (or is produced) at a constant rate
$p(\q)$.  When the robot footprint is covering $\q$, it consumes $Z$
at a constant rate $c(\q)$, so that when a point $\q$ is covered, the
net rate of decrease is $p(\q) - c(\q)$.  Thus, $Z$ evolves according
to the following differential equation (with initial conditions
$Z(\q,0)$ and $\theta(0)$):
\begin{equation}
\label{eq:Z_diffeq}
\dot Z(\q,t) =
\begin{cases}
  p(\q), & \text{if $\q\notin \BB\big(\theta(t)\big)$}, \\
  p(\q)-c(\q), &
  \text{if $\q \in \BB\big(\theta(t)\big)$ and $Z(\q,t) > 0$}, \\
  0, &
  \text{if $\q \in \BB\big(\theta(t)\big)$ and $Z(\q,t) = 0$}, 
\end{cases}
\end{equation}
where for each $\q\in Q$, we have $c(\q) > p(\q) > 0$.  In this paper,
we assume that we know the model parameters $p(\q)$ and $c(\q)$.  It
is reasonable to assume knowledge of $c(\q)$ since it pertains to the
performance of the robot.  For example, in an oil cleanup application,
the consumption rate of oil of the robot can be measured in a
laboratory environment or in field trials prior to control design.  As
for the production rate, $p(\q)$, this must be estimated from the
physics of the environment, from a human expert (e.g. an oil mining
engineer in the case of an oil well leak), or it can be measured in a
preliminary survey of the environment. However, the accuracy of the
model is not crucial, as we show analytically and in simulations that
our method has strong robustness with respect to errors in $p(\q)$.
The production function can also be estimated and incorporated into
the controller on-line, though we save the development of an on-line
strategy for future work.

With this notation, we can now formally define a \emph{persistent task}.
\begin{definition}[Persistent Tasks]
  A persistent task is a tuple $(\RR, \gamma, Q, p,c)$, where $\RR$ is
  the robot model, $\gamma$ is the curve followed by the robot, $Q$ is
  the set of points of interest, and $p$ and $c$ are the production
  and consumption rates of the field, respectively.
\end{definition}

\begin{figure}
  \centering
  \includegraphics[width=0.65\linewidth]{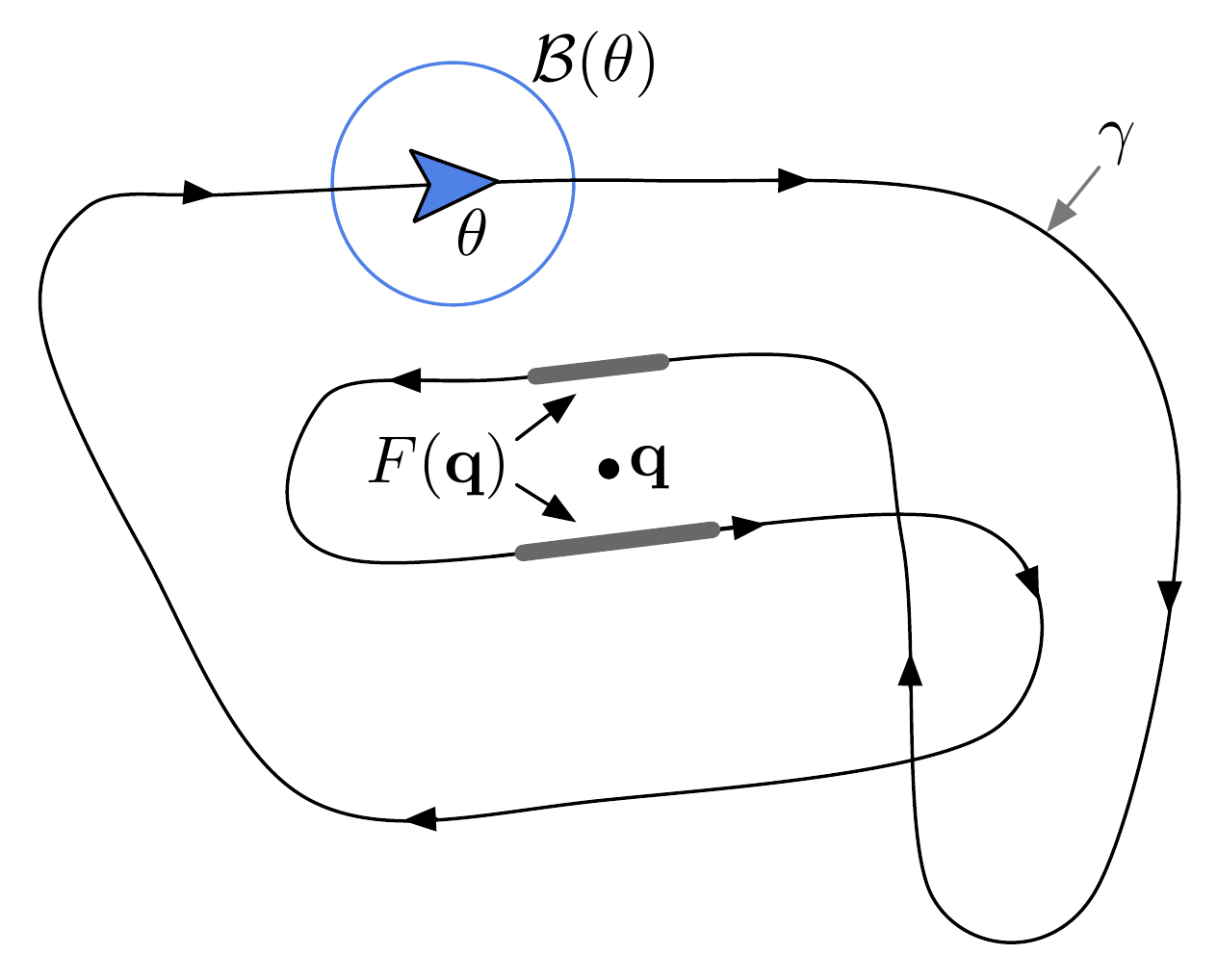}
  \caption{An illustration of a curve $\gamma$ followed by the robot.
    The robot is located at $\theta$ and has footprint $\BB(\theta)$.
    The set $F(\q)$ of robot positions $\theta$ for which the
    footprint covers $q$ are shown as thick grey segments of the curve.}
  \label{fig:curve_setup}
\end{figure}

In general, for a given persistent task, the commanded speed $v$ could
depend on the current position $\theta$, the field $Z$, the initial
conditions $\theta(0)$ and $Z(\q,0)$, and time $t$.  Thus, defining
the set of initial conditions as $\mathrm{IC}:=(\theta(0),Z(\q,0))$, a
general controller has the form $v(\theta,Z,\mathrm{IC},t)$.  For the
reader's convenience, the notation used in this paper is summarized in
Table~\ref{tab:parameters}.

\subsection{Field Stability and Feasibility}
\label{sec:stability_and_feasibility}

In this section we formalize the problem of stabilizing the field in a
persistent task.  As a first consideration, a suitable controller
should keep the field bounded everywhere, independent of the initial
conditions. This motivates the following definition of stability.

\begin{definition}[Field Stabilizing Controller] 
  \label{def:stability}
  A speed controller field stabilizes a persistent task if the field
  is always eventually bounded, independent of initial conditions.
  That is, if there exists a $Z_{\max}<+\infty$ such that for every
  $\q \in Q$ and initial condition $Z(\q,0)$ and $\theta(0)$, we have
  \[
  \limsup_{t\to+\infty} Z(\q,t) \leq Z_{\max}.
  \]
\end{definition}
Note that in this definition of stability, for every initial
condition, the field eventually enters the interval $[0,Z_{\max}]$.

There are some persistent tasks for which no controller is field
stabilizing.  This motivates the notion of \emph{feasibility}.

\begin{definition}[Feasible Persistent Task]
  A persistent task is feasible if there exists a field stabilizing
  speed controller.
\end{definition} 

As stated above, for a given persistent task, a general speed
controller can be written as $v(\theta,Z,\mathrm{IC},t)$.  However, in the
remainder of the paper we will focus on a small subset of speed
controllers which we call \emph{periodic position-feedback
  controllers}.  In these controllers, the speed only depends on the
robot's current position $\theta\in[0,1]$.  The controllers are
periodic in the sense that the speed at a point $\theta$ is the same
on each cycle of the path.  The controller can be written as
\[
v:[0,1] \to \real_{> 0},
\]
where each $\theta\in[0,1]$ is mapped to a speed $v(\theta)$
satisfying the bounds $\vmin(\theta)\leq v(\theta)\leq \vmax(\theta)$.
These controllers have the advantage that they do not require
information on the current state of the field $Z$, only its model
parameters $p(\q)$ and $c(\q)$.  While it may seem restrictive to
limit our controllers to this special form, the following result shows
that it is not.

\begin{proposition}[Periodic Position-Feedback Controllers]
  \label{prop:periodic_cont}
  If a persistent task can be stabilized by a general controller
  $v(\theta,Z,\mathrm{IC},t)$, then it can be stabilized by a periodic
  position-feedback controller $v(\theta)$.
\end{proposition}
The proof of Proposition~\ref{prop:periodic_cont} is given in
Appendix~\ref{sec:appendix}, and relies on the statement and proof of
the upcoming result in Lemma~\ref{lem:stab_cond}.  Therefore, we
encourage the reader to postpone reading the proof until after
Lemma~\ref{lem:stab_cond}.

We will now investigate conditions for a controller to be field
stabilizing and for a persistent task to be feasible. Let us define a
function which maps each point $\q \in Q$, to the curve positions
$\theta$ for which $\q$ is covered by the robot footprint.  To this
end, we define
\[
F(\q) := \setdef{\theta\in[0,1]}{\q\in\BB(\theta)}.
\]
An illustration of the curve, the robot footprint, and the set $F(\q)$
is shown in Figure~\ref{fig:curve_setup}. (In
Section~\ref{sec:simulations} we discuss how this set can be computed
in practice.)

Given a controller $\theta \mapsto v(\theta)$, we define two
quantities: 1) the cycle time, or period, $T$, and 2) the coverage
time per cycle $\tau(\q)$.  Since $v(\theta)>0$ for all $\theta$, the
robot completes one full cycle of the closed curve in time
\begin{equation}
  \label{eq:rev_time}
  T:= \int_0^{1}\frac{1}{v(\theta)} d\theta.
\end{equation}
During each cycle, the robot's footprint is covering the point $\q$
only when $\theta(t) \in F(\q)$. Thus the point $\q$ is covered for
\begin{equation}
\label{eq:cover_time}
\tau(\q):= \int_{F(\q)}\frac{1}{v(\theta)} d\theta,
\end{equation}
time units during each complete cycle.  

With these definitions we can give a necessary and sufficient
condition for a controller to stabilize a persistent task.  In
Section~\ref{sec:stability_and_optimality}, we develop a method for
testing if this condition can be satisfied by a speed controller.
\begin{lemma}[Stability condition]
  \label{lem:stab_cond}
  Given a persistent task, a controller $v(\theta)$ is field
  stabilizing if and only if
  \begin{equation}
   \label{eq:stability}
   c(\q)\int_{F(\q)} \frac{1}{v(\theta)} d\theta >
   p(\q) \int_0^1\frac{1}{v(\theta)} d\theta
  \end{equation}
  for every $\q\in Q$.  Applying the definitions
  in~\eqref{eq:rev_time} and~\eqref{eq:cover_time}, the condition can
  be expressed as $\tau(\q) > \big( p(\q)/c(\q)\big) T$.
\end{lemma}
The lemma has a simple intuition.  For stability, the field
consumption per cycle must exceed the field production per cycle, for
each point $\q\in Q$.  We now prove the result.
\begin{proof}
  Consider a point $\q\in Q$, and the set of curve positions $F(\q)$
  for which $\q$ is covered by the robot footprint.  Given the
  speed controller $v$, we can compute the cycle time $T$
  in~(\ref{eq:rev_time}).  Then, let us consider the change in the
  field from $Z(\q,t)$ to $Z(\q,t+T)$, where $t \geq 0$.

  Define an indicator function $\mathbf{I}:[0,1]\times Q\to \{0,1\}$
  as $\mathbf{I} (\theta,\q) = 1$ for $\theta \in F(\q)$ and $0$
  otherwise.  Then, from~(\ref{eq:Z_diffeq}) we have that
  \begin{align*}
    \dot Z(\q,t) \geq p(\q) - c(\q) \mathbf{I} (\theta(t),\q),
 \end{align*}
  for all values of $Z(\q,t)$, with equality if $Z(\q,t)>0$.
  Integrating the above expression over $[t,t+T]$ we see that
  \begin{align}
    Z(\q,t+T) - Z(\q,t)
    &\geq p(\q)T - c(\q) \int_t^{t+T}\mathbf{I}(\theta(\tau),\q) d\tau \nonumber\\
\label{eq:uncertainty_change}    
&= p(\q)T - c(\q)\tau(\q),
  \end{align}
  where $\tau(\q)$ is defined in~\eqref{eq:cover_time}, and the
  equality follows from the fact that $\mathbf{I}(\theta(t),\q)$ is
  simply an indicator function on whether or not the footprint is
  covering $\q$ at a given time.  From~\eqref{eq:uncertainty_change}
  we see that for the field to be eventually bounded by some
  $Z_{\max}$ for all initial conditions $Z(\q,0)$, we require that
  $\tau(\q) > p(\q)/c(\q) T$ for all $\q\in Q$.

  To see that the condition is also sufficient, suppose that $\tau(\q)
  > p(\q)/c(\q) T$.  Then there exists $\epsilon > 0$ such that $
  p(\q)T - c(\q)\tau(\q) = -\epsilon$.  If $Z(\q,t) > \big(c(\q) -
  p(\q) \big)T$, then the field at the point $\q\in Q$ is strictly
  positive over the entire interval $[t,t+T]$, implying that
   \[
   Z(\q,t+T) = Z(\q,t) -\epsilon.
  \]
  Thus, from every initial condition, $Z(\q,t)$ moves below
  $\big(c(\q) - p(\q) \big)T$.  Additionally, note that for each $\bar
  t$ in the interval $[t,t+T]$, we trivially have that $Z(\q,\bar t)
  \leq Z(\q,t) + p(\q) T$.  Thus, we have that there exists a finite
  time $\bar t$ such that for all $t\geq \bar t$,
  \[
  Z(\q,t) \leq \big(c(\q) - p(\q) \big)T +p(\q) T = c(\q) T.
  \]
  Since $Q$ is finite, there exists a single $\epsilon >0$ such that
  for every point $\q\in Q$ we have $\tau(\q) - p(\q)/c(\q) T >
  \epsilon$.  Hence, letting $Z_{\max} = \max_{\q\in Q} c(\q) T$, we
  see that $Z$ is stable for all $\q$, completing the proof.
\end{proof}

In the following sections we will address two problems, determining a
field stabilizing controller, and determining a minimizing controller,
defined as follows:
\begin{problem}[Persistent Task Metrics]
  \label{prob:PM}
  Given a persistent task, determine a periodic position-feedback
  controller $v:[0,1] \to \real_{>0}$ that satisfies the speed
  constraints (i.e., $v(\theta)
  \in\big[\vmin(\theta),\vmax(\theta)\big]$ for all $\theta\in[0,1]$),
  and
  \begin{enumerate}
  \item is field stabilizing; or
  \item minimizes the maximum steady-state field $\mathcal{H}(v)$:
    \[
    \mathcal{H}(v):=\max_{\q\in Q} \left(\limsup_{t\to+\infty} Z(\q,t)\right).
    \]
 \end{enumerate}
\end{problem}
In Section~\ref{sec:stability_and_optimality} we will show that by
writing the speed controller in terms of a set of basis functions,
problems (i) and (ii) can be solved using linear programs.  In
Section~\ref{sec:multi_robot} we will solve problem (i) for multiple
robots.

\section{Single Robot Speed Controllers: \\ Stability and Optimality}
\label{sec:stability_and_optimality}

In this paper we focus on a finite set of points of interest $Q
=\{\q_1,\ldots,\q_m\}$. These $m$ locations could be specific regions
of interest, or they could be a discrete approximation of the
continuous space obtained by, for example, laying a grid down on the
environment.  In the Section~\ref{sec:simulations} we will show
examples of both scenarios.  Our two main results are given in
Theorems~\ref{thm:lin_prog} and~\ref{thm:lin_prog_opt}, which show
that a field stabilizing controller, and a controller minimizing
$\mathcal{H}(v)$, can each be found by solving an appropriate linear
program.

To begin, it will be more convenient to consider the reciprocal
speed controller $v^{-1}(\theta):= 1/v(\theta)$, with its
corresponding constraints
\[
\frac{1}{\vmax(\theta)} \leq v^{-1}(\theta) \leq \frac{1}{\vmin(\theta)}.
\]
Now, our approach is to consider a finite set of basis functions
$\{\beta_1(\theta),\ldots,\beta_n(\theta)\}$.  Example basis functions
include (a finite subset of) the Fourier basis or Gaussian
basis~\cite{EWC:00}.  In what follows we will use rectangular
functions as the basis:
\begin{equation}
\label{eq:rect_basis}
\beta_j(\theta) = \begin{cases}
1, & \text{if $\theta\in[(j-1)/n,j/n)$} \\
0, & \text{otherwise},
\end{cases}
\end{equation}
for each $j\in\{1,\ldots,n\}$.  This basis, which provides a piecewise
constant approximation to a curve, has the advantage that we will
easily be able to incorporate the speed constraints $\vmin(\theta)$
and $\vmax(\theta)$ into the controller.

Then let us consider reciprocal speed controllers of the form
\begin{equation}
\label{eq:g_form}
v^{-1}(\theta) = \sum_{j=1}^n \alpha_j \beta_j(\theta),
\end{equation}
where $\alpha_1,\ldots,\alpha_n \in \real$ are free parameters that we
will use to optimize the speed controller.  A rich class of functions
can be represented as a finite linear combination of basis functions,
though not all functions can be represented this way.  Limiting our
speed controller to a linear parametrization allows us to find an
optimal controller within that class, while preserving enough
generality to give complex solutions that would be difficult to find
in an ad hoc manner.  In the following subsection we will consider the
problem of synthesizing a field stabilizing controller.

\subsection{Synthesis of a Field Stabilizing Controller}
\label{sec:stabilizing_speed}

In this section we will show that a field stabilizing speed controller of
the form~\eqref{eq:g_form} can be found through the solution of a
linear program.  This result is summarized in
Theorem~\ref{thm:lin_prog}. We remind the reader that A summary of the
mathematical symbols and their definitions is shown in
Table~\ref{tab:parameters}.

\begin{table*}[htb]
\small
\centering %
\caption{Table of Symbols}%
\label{tab:parameters}
\begin{tabular}{|c|l|}
  \hline
  {\bf Symbol} & {\bf Definition} \\
  \hline
  $\env$ & the environment. \\
  $Q$ & a set of points of interest in the environment $\env$. \\
  $\q$ & a point of interest in $Q$. \\
  $\gamma$ & a curve followed by the robot, and parametrized by $\theta$. \\
  $v(\theta)$ & the speed of the robot along the curve $\gamma$. \\
 $\RR$ & the robot model, consisting of minimum and maximum speeds
  $\vmin(\theta)$ and $\vmax(\theta)$, and a robot footprint $\BB(\theta)$. \\
  $\BB(\theta)$ & The set of points in the environment covered by the
  robot's footprint when positioned at $\theta$. \\
  $F(\q)$ & the values of $\theta\in[0,1]$ along the curve $\gamma$ for which the point $\q$
  is covered by the footprint.  \\
  $Z(\q,t)$ & the field (or accumulation function) at point $\q\in Q$ at time $t$. \\
  $p(\q)$ & the rate of production of the field $Z$ at
  point $\q$. \\
  $c(\q)$ & the rate of consumption of the field $Z$ when
  $\q$ is covered by the robot footprint. \\
  $\beta(\theta)$ & a basis function. \\
  $\alpha$ & a basis function coefficient, or parameter. \\
  $T$ & the time to complete one full cycle of the curve. \\
  $\tau(\q)$ & the amount of time the point $\q\in Q$ is covered per
  cycle. \\
  $\mathcal{H}(v)$ & cost function~\eqref{eq:H_cost} giving the steady-state maximum
  value of $Z(\q,t)$ \\
  \hline
\end{tabular}
\end{table*}
To begin, let us consider reciprocal speed controllers in the form
of~\eqref{eq:g_form}.  Then for $\q_i \in Q$, the stability condition
in Lemma~\ref{lem:stab_cond} becomes
\begin{equation*}
  \sum_{j=1}^n \alpha_j \int_{F(\q_i)} \beta_j(\theta)d\theta >
  \frac{p(\q_i)}{c(\q_i)}  \sum_{j=1}^n \alpha_j \int_0^1 \beta_j(\theta)d\theta
\end{equation*}
Rearranging, we get
\[
\sum_{j=1}^n \alpha_j K(\q_i,\beta_j) > 0,
\]
where we have defined
\begin{equation}
\label{eq:C_def}
K(\q_i,\beta_j) := \int_{F(\q_i)} \beta_j(\theta)d\theta -
  \frac{p(\q_i)}{c(\q_i)}\int_0^1 \beta_j(\theta) d\theta.
\end{equation}
Finally, to satisfy the speed constraints we have that
\begin{equation}
\label{eq:speed_constraints}
\frac{1}{\vmax(\theta)} \leq \sum_{j=1}^n \alpha_j \beta_j(\theta)
\leq \frac{1}{\vmin(\theta)}
\end{equation}
For the rectangular basis in~(\ref{eq:rect_basis}), the speed
constraints become
\[
\frac{1}{\vmax(j)} \leq \alpha_j \leq \frac{1}{\vmin(j)},
\]
where
\begin{align*}
\vmax(j) &= \inf_{\theta\in[(j-1)/n,j/n)} \vmax(\theta), \; \text{and} \\
\vmin(j) &= \sup_{\theta\in[(j-1)/n,j/n)} \vmin(\theta). 
\end{align*}
Thus, we obtain the following result.
\begin{theorem}[Existence of a Field Stabilizing Controller]
  \label{thm:lin_prog}
  A persistent task is stabilizable by a speed controller of the
  form~\eqref{eq:g_form} if and only if the following linear program
  is feasible:
  \begin{align*}
    \text{minimize} \;\;& 0 & \\
    \text{subject to} \;\;& \sum_{j=1}^n \alpha_j K(\q_i,\beta_j)> 0
    &\forall\; i\in\{1,\ldots,m\}\\
    & \frac{1}{\vmax(j)} \leq \alpha_j \leq \frac{1}{\vmin(j)}, &
    \forall\; j\in\{1,\ldots,n\},
  \end{align*}
  where $K(\q_i,\beta_j)$ is defined in~\eqref{eq:C_def}, and
  $\alpha_1,\ldots,\alpha_n$ are the optimization variables.
\end{theorem}
Hence, we can solve for a field stabilizing controller using a simple linear
program.  The program has $n$ variables (one for each basis function
coefficient), and $2n+m$ constraints (two for each basis function
coefficient, and one for each point of interest in $Q$).  One can
easily solve linear programs with thousands of variables and
constraints~\cite{DGL:84}. Thus, the problem of computing a
field stabilizing controller can be solved for finely discretized
environments with thousands of basis functions.  Note that in the
above lemma, we are only checking feasibility, and thus the cost
function in the optimization is arbitrary.  For simplicity we write
the cost as $0$.

In Theorem~\ref{thm:lin_prog} the cost is set to $0$ to highlight the
feasibility constraints.  However, in practice, an important
consideration is \emph{robustness} of the controller to uncertainty
and error in the model of the field evolution, and in the motion of
the robot.  Robustness of this type can be achieved by slightly
altering the above optimization to maximize the \emph{stability
  margin}.  This is outlined in the following result.

\begin{corollary}[Robustness via Maximum Stability Margin]
\label{rmk:MinCycleController}
The optimization
  \begin{align*}
    \text{maximize} \;\;& B & \\
   & \sum_{j=1}^n \alpha_j K(\q_i,\beta_j) \ge B
    & \forall\; i\in\{1,\ldots,m\}\\
    & \frac{1}{\vmax(j)} \leq \alpha_j \leq \frac{1}{\vmin(j)}, &
    \forall\; j\in\{1,\ldots,n\},
  \end{align*}
  where $\alpha_1,\ldots,\alpha_n$ and $B$ are the optimization
  variables, yields a speed controller which maximizes the stability
  margin, $\min_{\q_i\in Q}\sum_{j=1}^n \alpha_j K(\q_i,\beta_j)$.
  The controller
  \begin{enumerate}
  \item has the largest decrease in $Z(\q_i,t)$ per cycle, and thus
    achieves steady-state in the minimum number of cycles.
  \item is robust to errors in estimating the field production rate.
    If the robot's estimate of the production rate at a field point
    $\q_i\in Q$ is $\bar p(\q_i)$, and the true value is $p(\q_i) \leq
    \bar p(\q_i) +\epsilon$, then the field is stable provided that
    for each $\q_i \in Q$
    \[
    \epsilon < B \cdot c(\q_i) \left(\sum_{j=1}^n \alpha_j \int_0^1
    \beta_j(\theta)d\theta\right)^{-1}.
    \]
  \end{enumerate}
\end{corollary}
\begin{proof}
  The first property follows directly from the fact that $\sum_{j=1}^n
  \alpha_j K(\q_i,\beta_j)$ is the amount of decrease of the field at
  point $\q_i$ in one cycle.

  To see the second property, suppose that the optimization is based
  on a production rate $\bar p(\q_i)$ for each $\q_i\in Q$, but the
  true production rate is given by $p(\q_i) := \bar p(\q_i) +
  \epsilon$, where $\epsilon >0$.  In solving the optimization above,
  we obtain a controller with a stability margin of $B >0$. For the
  true system to be stable we require that for each $\q_i\in Q$,
\[
 \sum_{j=1}^n \alpha_j \int_{F(\q_i)} \beta_j(\theta)d\theta >
  \frac{\bar p(\q_i)+\epsilon}{c(\q_i)}  \sum_{j=1}^n \alpha_j \int_0^1 \beta_j(\theta)d\theta.
\]   
This can be rewritten as 
\[
\sum_{j=1}^n \alpha_j K(\q_i,\beta_j) > \frac{\epsilon}{c(\q_i)}
\sum_{j=1}^n \alpha_j \int_0^1 \beta_j(\theta)d\theta,
\]
from which we see that the true field is stable provided that
$\epsilon < B c(\q_i) /\big(\sum_{j=1}^n \alpha_j \int_0^1
\beta_j(\theta)d\theta\big)$ for each $\q_i\in Q$.
\end{proof}

\begin{remark}[Alternative Basis Functions]
  Given an arbitrary set of basis functions, the inequalities
  in~\eqref{eq:speed_constraints} yield an infinite (and uncountable)
  number of constraints, one for each $\theta\in[0,1]$.  This means
  that for some basis functions the linear program in
  Theorem~\ref{thm:lin_prog} can not be solved exactly.  However, in
  practice, an effective method is to simply enforce the constraints
  in~\eqref{eq:speed_constraints} for a finite number of $\theta$
  values, $\theta_1,\ldots,\theta_w$.  Each $\theta_i$ generates two
  constraints in the optimization.  Then, we can tighten each
  constraint by $\xi >0$, yielding $1/v_{\max}(\theta_i) +\xi \leq
  \sum_j \alpha_j \beta_j(\theta_i) \leq 1/v_{\min}(\theta_i)-\xi$ for
  each $i\in\{1,\ldots,w\}$.  Choosing the number of $\theta$ values
  $w$ as large as possible given the available computational
  resources, we can then increase $\xi$ until the controller satisfies
  the original speed constraints.  \oprocend
\end{remark}

\subsection{Synthesis of an Optimal Controller}
\label{sec:optimal_speed}

In this section we look at Problem~\ref{prob:PM} (ii), which is to
minimize the maximum value attained by the field over the finite
region of interest $Q$.  That is, for a given persistent task, our goal is to
minimize the following cost function,
\begin{equation}
\label{eq:H_cost}
\mathcal{H}(v) = \max_{\q\in Q} \left(\limsup_{t\to+\infty} Z(\q,t) \right)
\end{equation}
over all possible speed controllers $v$.  At times we will refer to the
maximum steady-state value for a point $\q$ using a speed controller $v$
as 
\[
\mathcal{H}(\q,v):= \limsup_{t\to+\infty} Z(\q,t)
\]

Our main result of this section, Theorem~\ref{thm:lin_prog_opt}, is
that $\mathcal{H}(v)$ can be minimized through a linear program.
However, we must first establish intermediate results.  First we show
that if $v$ is a field stabilizing controller, then for every initial
condition there exists a finite time $t^*$ such that $Z(\q,t) \leq
\mathcal{H}(v)$ for all $t\geq t^*$.

\begin{proposition}[Steady-State Field]
  \label{prop:steady_state}
  Consider a feasible persistent task and a field stabilizing speed controller. Then,
  there is a steady-state field
  \[
  \bar Z : Q \times [0,1] \to \real_{\geq 0},
  \]
  satisfying the following statements for each $\q\in Q$:
  \begin{enumerate}
  \item for every set of initial conditions $\theta(0)$ and $Z(\q,0)$,
    there exists a time $t^* \geq 0$ such that
    \[
    Z(\q,t) = \bar Z\big(\q,\theta(t)\big),
    \]
    for all $t \geq t^*$.
  \item there exists at least one $\theta \in [0,1]$ such that $\bar
    Z(\q,\theta) = 0$.
  \end{enumerate}
\end{proposition}
From the above result we see that from every initial condition, the
field converges in finite time to a steady-state $\bar Z(\q,\theta)$.
In steady-state, the field $Z(\q,t)$ at time $t$ depends only on
$\theta(t)$ (and is independent of $Z(\q,0)$).  Each time the robot is
located at $\theta$, the field is given by $\bar Z(\q,\theta)$.
Moreover, the result tells us that in steady-state there is always a
robot position at which the field is reduced to zero.  In order to
prove Proposition~\ref{prop:steady_state} we begin with the following
lemma.  Recall that the cycle-time for a speed controller $v$ is $T
:=\int_0^11/v(\theta) d\theta$.

\begin{lemma}[Field Reduced to Zero]
  \label{lem:reduce_to_zero}
  Consider a feasible persistent task and a field stabilizing speed
  controller.  For every $\q\in Q$ and every set of initial conditions
  $Z(\q,0)$ and $\theta(0)$, there exists a time $t^* > T$ such that
  \begin{equation}
    \label{eq:periodic_t}
  Z(\q,t^* + a T) = 0,
  \end{equation}
  for all non-negative integers $a$.
\end{lemma}
\begin{proof}
  Consider any $\q\in Q$, and initial conditions $Z(\q,0)$ and
  $\theta(0)$, and suppose by way of contradiction that the speed
  controller is stable but $Z(\q,t) > 0$ for all $t > T$.  From
  Lemma~\ref{lem:stab_cond}, if the persistent task is stable, then $c(\q)
  \tau(\q) > p(\q) T$ for all $\q$.  Thus, there exists $\epsilon > 0$
  such that $c(\q) \tau(\q) - p(\q) T > \epsilon$ for all $\q\in
  Q$. From the proof of Lemma~\ref{lem:stab_cond}, we have that
  \[
  Z(\q,t+T) - Z(\q,t) = p(\q) T - c(\q) \tau(\q) = -\epsilon.
  \]
  Therefore, given $Z(\q,0)$, we have that $Z(\q,t^*) = 0$ for some
  finite $t^* > T$, a contradiction.

  Next we will verify that if $Z(\q,t^*) = 0$ for some $t^* > T$, then
  $Z(\q,t^*+T) = 0$.  To see this, note that the differential
  equation~\eqref{eq:Z_diffeq} is piecewise constant.  Given a speed
  controller $v(\theta)$, the differential equation is time-invariant,
  and admits unique solutions.

  Based on this, consider two initial conditions
  for~\eqref{eq:Z_diffeq}, 
  \[
  Z_1(\q,0) := Z(\q,t^*-T) \geq 0, \quad 
  \theta_1(0) := \theta(t^*-T) = \theta(t^*),
  \]
  and 
  \[
  Z_2(\q,0):=Z(\q,t^*) = 0, \quad \theta_2(0) := \theta(t^*).
  \]
  Since~\eqref{eq:Z_diffeq} is time-invariant, we have that $Z_1(\q,T)
  = Z(\q,t^*) = 0$, and $Z_2(\q,T) = Z(\q,t^*+T)$.  In addition, by
  uniqueness of solutions, we also know that $Z_1(\q,0) \geq
  Z_2(\q,0)$ implies that $Z_1(\q,T) \geq Z_2(\q,T)$.  Thus, we have
  that $Z(\q,t^*)= 0 \geq Z(\q,t^*+T)$, proving the desired result.
\end{proof}
The previous lemma shows that from every initial condition there
exists a finite time $t^*$, after which the field at a point $\q$ is
reduced to zero in each cycle.  With this lemma we can prove
Proposition~\ref{prop:steady_state}.

\begin{proof}[Proof of Proposition~\ref{prop:steady_state}] 
  In Lemma~\ref{lem:reduce_to_zero} we have shown that for every set
  of initial conditions $Z(\q,0)$, $\theta(0)$, there exists at time
  $t^* > T$ such that $Z(\q,t^* +aT) = 0$ for all non-negative
  integers $a$. Since $T$ is the cycle-time for the robot, we also
  know that $\theta(t^*+aT) = \theta(t^*)$ for all $a$.
  Since~\eqref{eq:Z_diffeq} yields unique
  solutions,~\eqref{eq:periodic_t} uniquely defines $Z(\q,t)$ for all
  $t \geq t^*$, with
  \[
  Z(\q,t+T) = Z(\q,t) \quad \text{for all $t \geq t^*$.} 
  \]
  Hence, we can define the steady-state profile $\bar Z(\q,\theta)$ as
  \[
  \bar Z\big(\q,\theta(t)\big) := Z(\q,t) \quad \text{for all
    $t\in[t^*,t^*+T)$}.
  \]

  Finally, we need to verify that $\bar Z(\q,\theta)$ is independent
  of initial conditions.  To proceed, suppose by way of contradiction
  that there are two sets of initial conditions $\theta_1(0)$,
  $Z_1(\q,0)$, and $\theta_2(0)$, $Z_2(\q,0)$ which yield different
  steady-state fields $\bar Z_1(\q,\theta)$ and $\bar Z_2(\q,\theta)$.
  That is, there exists $\tilde \theta$ such that $\bar Z_1(\q,\tilde
  \theta) \neq \bar Z_2(\q,\tilde \theta)$. Without loss of
  generality, assume that $\bar Z_1(\q,\tilde \theta) > \bar
  Z_2(\q,\tilde \theta)$.  To obtain a contradiction, we begin by
  showing that this implies $\bar Z_1(\q,\theta) \geq \bar
  Z_2(\q,\theta)$ for all $\theta$.  Note that $Z_1$ and $Z_2$ reach
  their steady-state profiles $\bar Z_1$ and $\bar Z_2$ in finite
  time.  Thus, there exist times $t_1,t_2 \geq 0$ for which
  $\theta_1(t_1) = \theta_2(t_2) = \tilde \theta$, and $Z_1(\q,t_1) =
  \bar Z_1(\q,\tilde \theta)$, and $Z_2(\q,t_2) = \bar Z_2(\q,\tilde
  \theta)$.  Since $Z_1(\q,t_1) > Z_2(\q,t_2)$, and since $Z$ is a
  continuous function of time, either i) $Z_1(\q,t_1 + t) > Z_2(\q,t_2
  + t)$ for all $t\geq 0$, or ii) there exists a time $\bar t > 0$ for
  which $Z_1(\q,t_1 + \bar t) = Z_2(\q,t_2 + \bar t)$, which by
  uniqueness of solutions implies $Z_1(\q,t_1 + t) = Z_2(\q,t_2 + t)$
  for all $t\geq \bar t$.  Thus, $Z_1(\q,t_1 + t) \geq Z_2(\q,t_2 +
  t)$ for all $t \geq 0$, implying that $\bar Z_1(\q,\theta) \geq \bar
  Z_2(\q,\theta)$ for all $\theta$.

  From Lemma~\ref{lem:reduce_to_zero}, there exists a $\bar \theta$
  for which $\bar Z_1(\q,\bar \theta) = 0$.  Since, $\bar
  Z_1(\q,\theta) \geq \bar Z_2(\q,\theta)$ for all $\theta$, we must
  have that $\bar Z_2(\q,\bar\theta)=0$.  However, the value of $Z_1$
  and $Z_2$ at $\bar \theta$ uniquely defines $\bar Z_1$ and $\bar
  Z_2$ for all $\theta$, implying that $\bar Z_1(\q,\theta) = \bar
  Z_2(\q,\theta)$, a contradiction.
\end{proof}

From Proposition~\ref{prop:steady_state} we have shown the existence of a
steady-state field $\bar Z(\q,\theta)$ that is independent of initial
conditions $Z(\q,0)$ and $\theta(0)$.

Now, consider a point $\q\in Q$ and a field stabilizing speed controller
$v(\theta)$, and let us solve for its steady-state field $\bar
Z(\q,\theta)$.  To begin, let us write $F(\q)$ (the set of $\theta$
values for which the point $\q$ is covered by the footprint) as a
union of disjoint intervals
\begin{equation}
\label{eq:F_q}
F(\q) = [x_1,y_1] \cup [x_2,y_2] \cup \cdots \cup [x_{\ell},y_{\ell}],
\end{equation}
where $\ell$ is a positive integer, and $y_k > x_k > y_{k-1}$ for each
$k\in\{1,\ldots,\ell\}$.\footnote{Note that the number of intervals
  $\ell$, and the points $x_1,\ldots,x_{\ell}$ and
  $y_1,\ldots,y_{\ell}$ are a function of $\q$.  However, for
  simplicity of notation, we will omit writing the explicit
  dependence.} Thus, on the intervals $[x_k,y_k]$ the point $\q$ is
covered by the robot footprint, and on the intervals $[y_k,x_{k+1}]$,
the point $\q$ in uncovered.  As an example, in
Figure~\ref{fig:curve_setup}, the set $F(\q)$ consists of two
intervals, and thus $\ell =2$.  An example of a speed controller and
an example of a set $F(\q)$ are shown in
Figures~\ref{fig:speed_profile} and~\ref{fig:steady_state}.

\begin{figure*}
  \centering
  \subfloat[A sample reciprocal speed controller $v(\theta)$]{%
    \includegraphics[width=0.32\linewidth]{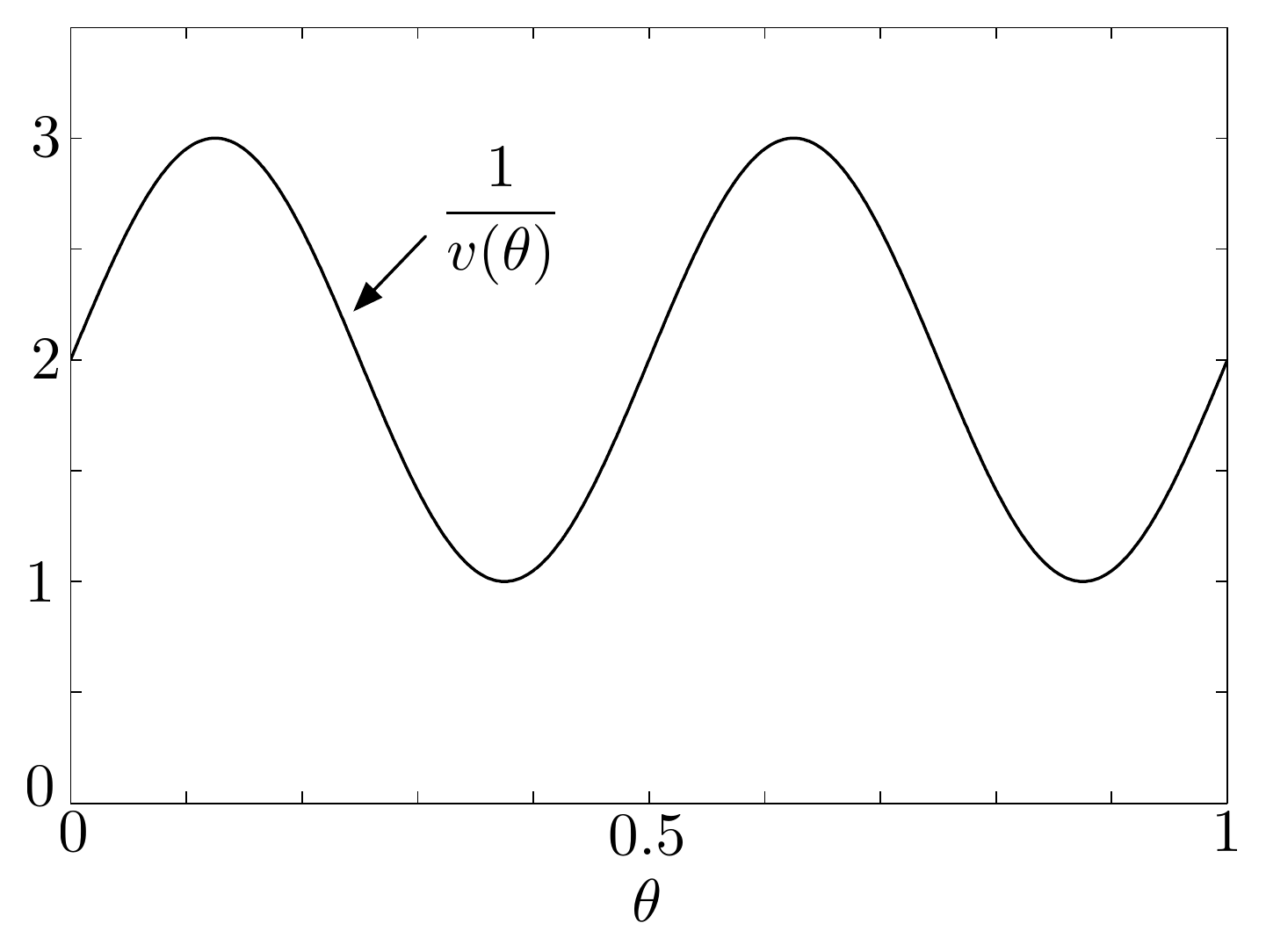}
    \label{fig:speed_profile}
}%
\hfill
 \subfloat[The steady-state field $\bar Z(\q,\theta)$.  The set
  $F(\q)$ consists of three intervals which are shaded on the
  $\theta$-axis.  The steady-state profile is increasing outside of
  $F(\q)$ and decreasing inside $F(\q)$.]{%
\includegraphics[width=0.32\linewidth]{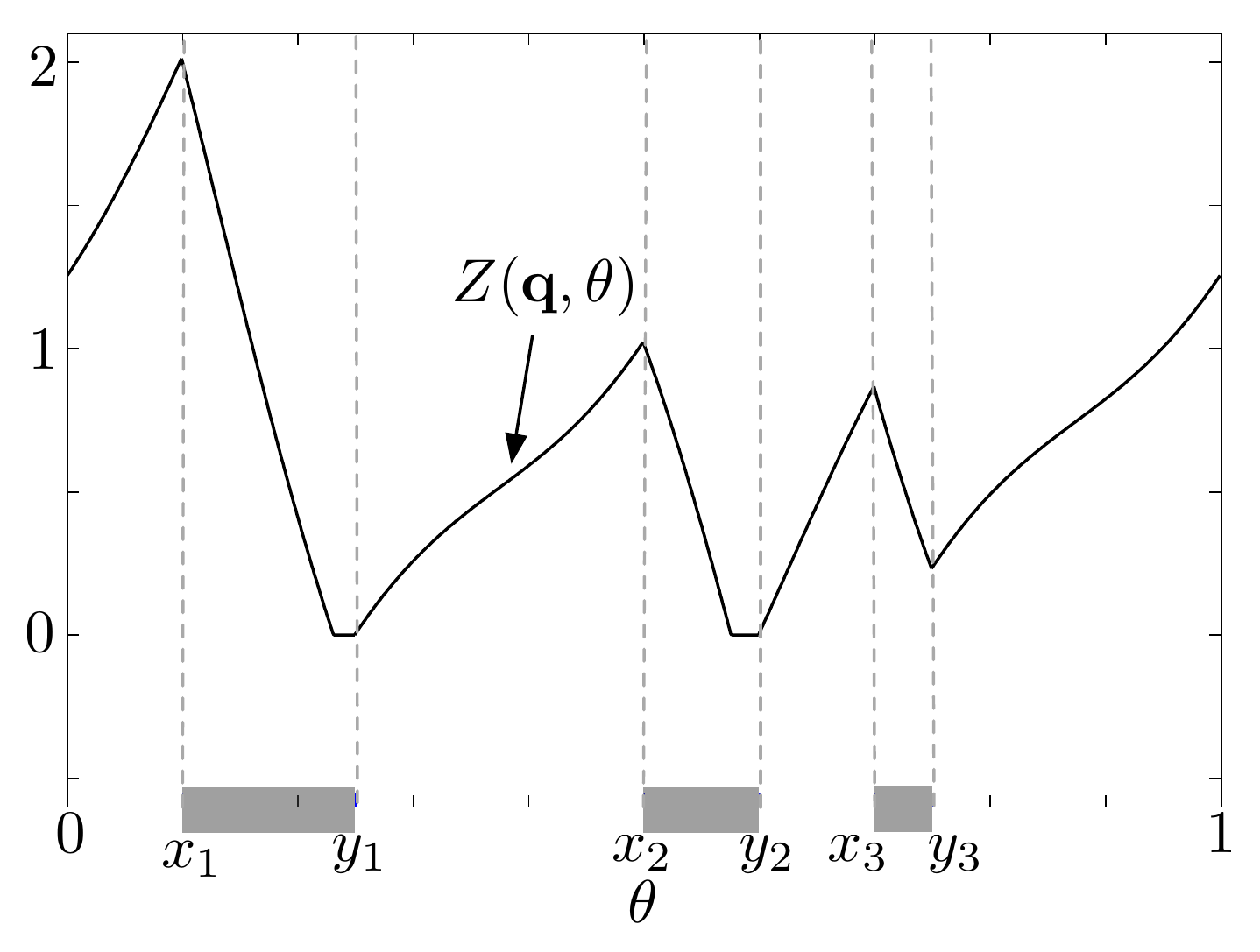} 
\label{fig:steady_state}  
}%
\hfill %
\subfloat[The calculation of quantities $N_{1,2}$, $N_{2,3}$ and
$N_{3,1}$.  The values represent the maximum reduction from $y_{k-1}$
to $y_{k}$.  Thus, $N_{1,2},N_{3,1} <0$ while $N_{2,3} >0$.]{%
    \includegraphics[width=0.32\linewidth]{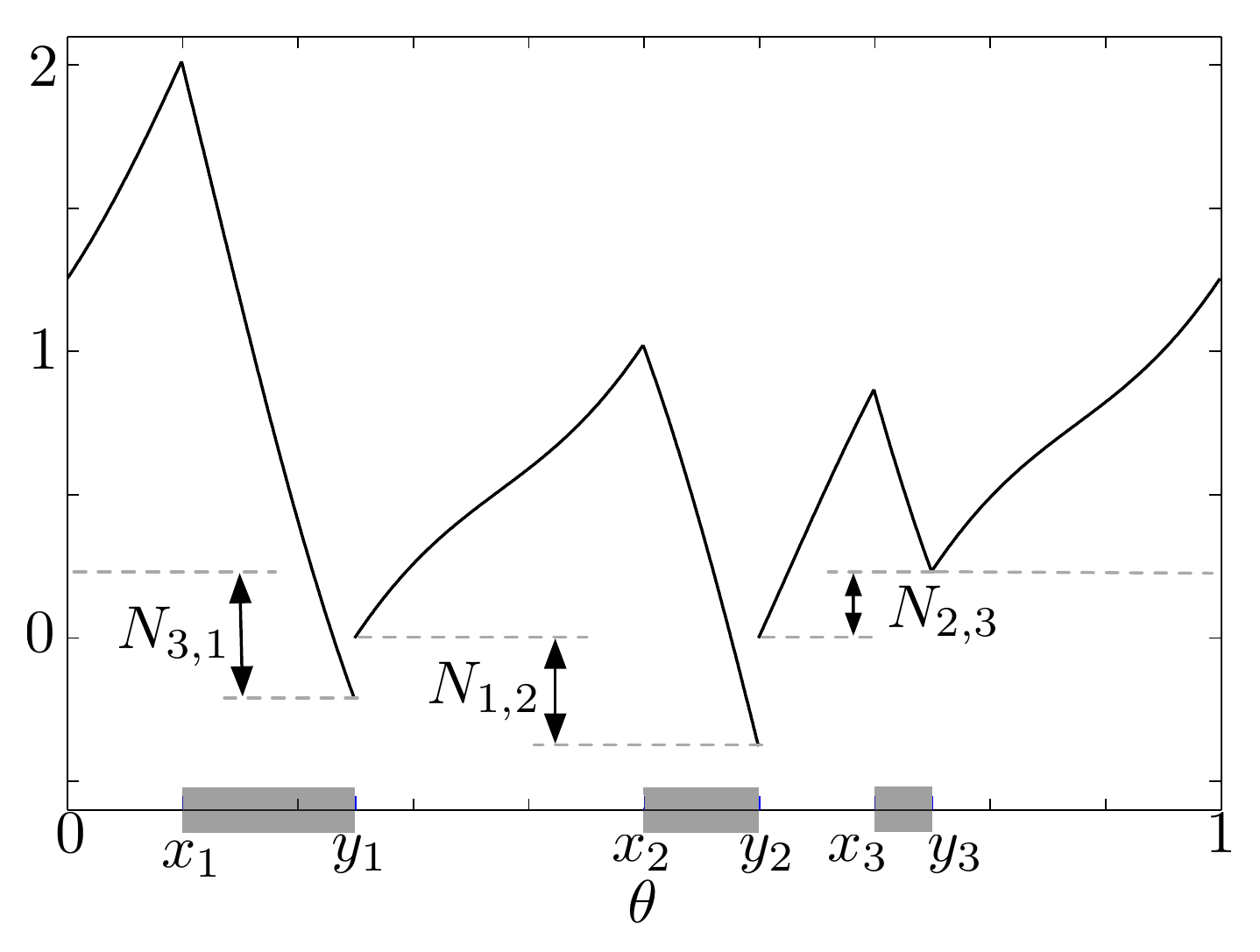} %
   \label{fig:steady_state_below}
  }%
  \caption{The steady-state field for a point $\q\in Q$.  The set of
    $\theta$ values for which $\q$ is covered is given by $F(\q) =
    [x_1,y_1]\cup [x_2,y_2]\cup[x_3,y_3]$.  The field is produced at a
    rate $p(\q)=3$, and is consumed by the footprint at a rate
    $c(\q)=8.5$.}
  \label{fig:steady_state_Z}
\end{figure*}

From differential equation~\eqref{eq:Z_diffeq} we can write
\begin{align}
  \bar Z(\q,x_k) &= \bar Z(\q,y_{k-1}) + p(\q)
  \int_{y_{k-1}}^{x_k}\frac{d\theta}{v(\theta)} \label{eq:x_to_y}\\
  \bar Z(\q,y_k) &= \left(\bar Z(\q,x_{k}) +\big(p(\q) - c(\q) \big)
    \int_{x_{k}}^{y_k}\frac{d\theta}{v(\theta)}\right)^+ \label{eq:y_to_x},
\end{align}
where for $z\in\real$, we define $(z)^+:= \max\{z,0\}$.  Combining
equations~\eqref{eq:x_to_y} and~\eqref{eq:y_to_x} we see that
\begin{multline}
\label{eq:recursive_Z}
\bar Z(\q,y_k) = \left(\bar Z(\q,y_{k-1}) + p(\q)
  \int_{y_{k-1}}^{y_k}\frac{d\theta}{v(\theta)} \right.\\ -\left. c(\q)
  \int_{x_{k}}^{y_k}\frac{d\theta}{v(\theta)}\right)^+.
\end{multline}
For each $b\in\{1,\ldots,\ell\}$, let us define\footnote{In this
  definition, and in what follows, addition and subtraction in the
  indices is performed modulo $\ell$.  Therefore, if $k=1$, then
  $N_{k-1,k} = N_{\ell,1}$.}
\begin{equation}
  \label{eq:N_k-b_k}
N_{k-b,k}(\q) := p(\q)\int_{y_{k-b}}^{y_{k}}\frac{d\theta}{v(\theta)}
-c(\q) \sum_{w=0}^{b-1}
\int_{x_{k-w}}^{y_{k-w}}\frac{d\theta}{v(\theta)}.
\end{equation}
Note that we can write
\[
\bar Z(\q,y_k) = \left( \bar Z(\q,y_{k-1}) + N_{k-1,k}(\q) \right)^+.
\]
and, from~\eqref{eq:recursive_Z} we have
\begin{equation}
\label{eq:Z_upper}
\bar Z(\q,y_k) \geq \bar Z(\q,y_{k-b}) + N_{k-b,k}(\q).
\end{equation}
Moreover,
\begin{multline}
\label{eq:Z_upper_eq}
\bar Z(\q,y_k) = \bar Z(\q,y_{k-b}) + N_{k-b,k}(\q), \\ 
\text{if $Z(\q,y_{k-j}) > 0$ for all $j\in\{1,\ldots,b-1\}$}.
\end{multline}
Thus, we see that the quantity $N_{k-b,k}(\q)$ gives the maximum
reduction in the field between $\theta = y_{k-b}$ and $\theta =
y_{k}$.  An example for $b=1$ is shown in
Figure~\ref{fig:steady_state_below}.  With these definition, we can
characterize the steady-state field at the points $y_k$.

\begin{lemma}[Steady-State Field at Points $y_k$]
  Given a feasible persistent task and a field stabilizing speed controller, consider a
  point $\q\in Q$ and the set $F(\q) = \cup_{k=1}^{\ell} [x_k,y_k]$.
  Then, for each $k\in\{1,\ldots,\ell\}$ we have
  \[
  \bar Z(\q,y_k) = \max_{b\in\{0,\ldots,\ell-1\}} N_{k-b,k}(\q),
  \]
  where $N_{k-b,k}(\q)$ is defined in~\eqref{eq:N_k-b_k} and
  $N_{k,k}(\q) := 0$.
\end{lemma}
\begin{proof}
  Let us fix $k\in\{1,\ldots,\ell\}$.  Given a field stabilizing controller,
  Proposition~\ref{prop:steady_state} tells us that there exists $\theta$
  such that $\bar Z(\q,\theta) = 0$.  It is clear that this must occur
  for some $\theta \in F(\q)$.  Therefore,
  \[
  \bar Z(\q,y_j) = 0 \quad \text{for some $j\in\{1,\ldots,\ell\}$}.
  \]

  Let $b$ be the smallest non-negative integer such that
  $\bar Z(\q,y_{k-b}) = 0$.  By \eqref{eq:Z_upper_eq} we have
  \[
  \bar Z(\q,y_k) = N_{k-b,k}(\q) \geq 0.
  \]
  If $b = 0$, then the previous equation simply states that $Z(\q,y_k)
  = 0$.  Now, if $N_{k-d,k}(\q) \leq N_{k-b,k}(\q)$ for all
  $d\in\{0,\ldots,\ell\}$, then $\bar Z(\q,y_k) =
  \max_{b\in\{0,\ldots,\ell-1\}} N_{k-b,k}(\q)$, and we have completed the
  proof.

  Suppose by way of contradiction that there is
  $d\in\{0,\ldots,\ell-1\}$ for which $N_{k-d,k}(\q) > N_{k-b,k}(\q)$.
  From~\eqref{eq:Z_upper} we have
  \[
  \bar Z(\q,y_k) \geq \bar Z(\q,y_{k-d}) + N_{k-d,k}(\q) \geq N_{k-d,k}(\q), 
  \]
  where the second inequality comes from the fact that $\bar
  Z(\q,y_{k-d}) \geq 0$. However $\bar Z(\q,y_k) = N_{k-b,k}(\q)$,
  implying that $N_{k-d,k}(\q) \leq N_{k-b,k}(\q)$, a contradiction.
\end{proof}

The above lemma gives the value of the field in steady-state at each
end point $y_k$. The field decreases from $x_k$ to $y_k$ (since these
are the $\theta$ values over which the point $\q$ is covered), and
then increases from $y_k$ to $x_{k+1}$.  Therefore, the maximum
steady-state value is attained at an end point $x_k$ for some $k$.
For example, in Figure~\ref{fig:steady_state}, the maximum is attained
at the point $x_1$.  However, the value at $x_k$ can be easily
computed from the value at $y_{k-1}$ using~\eqref{eq:x_to_y}:
\[
\bar Z(\q,x_{k+1}) = \max_{b\in\{0,\ldots,\ell-1\}} N_{k-b,k}(\q) + p(\q)
\int_{y_{k}}^{x_{k+1}}\frac{d\theta}{v(\theta)}.
\]
From this we obtain the following result.
\begin{lemma}[Steady-State Upper Bound]
  \label{lem:steady-state_upper}
  Given a field stabilizing speed controller $v$, the maximum steady-state
  field at $\q\in Q$ (defined in~\eqref{eq:H_cost}) satisfies
  \[
  \mathcal{H}(\q,v) =
  \max_{\substack{k\in\{1,\ldots,\ell\}\\b\in\{0,\ldots,\ell-1\}}} X_{k,b}(\q),
  \]
  where
  \[
  X_{k,b}(\q) = 
    p(\q)\int_{y_{k-b}}^{x_{k+1}}\frac{d\theta}{v(\theta)} -c(\q)
   \sum_{w=0}^{b-1} \int_{x_{k-w}}^{y_{k-w}}\frac{d\theta}{v(\theta)},
  \]
  and $F(\q) = \cup_{k=1}^{\ell} [x_k,y_k]$ with $y_k > x_k >
  y_{k-1}$ for each $k$.
\end{lemma}

The above lemma provides a closed form expression (albeit quite
complex) for the largest steady-state value of the field.  Thus,
consider speed controllers of the form
\[
v^{-1}(\theta) = \sum_{j=1}^n \alpha_j \beta_j(\theta),
\]
where $\beta_1,\ldots,\beta_n$ are basis functions (e.g., the
rectangular basis). For a finite field $Q = \{\q_1,\ldots,\q_m\}$, the
terms $N_{k-b,k}(\q_i)$ can be written as
\[
X_{k,b}(\q_i) = \sum_{j=1}^n \alpha_j X_{k,b}(\q_i,\beta_j),
\]
where
\begin{multline}
\label{eq:X_kb}
  X_{k,b}(\q_i,\beta_j) :=
  \\p(\q)\int_{y_{k-b}}^{y_{k}}\beta_j(\theta)d\theta -c(\q)
  \sum_{w=0}^{b-1}
  \int_{x_{k-w}}^{y_{k-w}}\beta_j(\theta)d\theta.
\end{multline}
With these definitions we can define a linear program for minimizing
the maximum of the steady-state field.  We will write $\ell(\q)$ to
denote the number of disjoint intervals on the curve $\gamma$ over
which the point $\q$ is covered, as defined in~\eqref{eq:F_q}.

\begin{theorem}[Minimizing the Steady-State Field]
    \label{thm:lin_prog_opt}
    Given a feasible persistent task, the solution to the following linear
    program yields a speed controller $v$ of the form~\eqref{eq:g_form} that
    minimizes the maximum value of the steady-state
    field~$\mathcal{H}(v)$.
  \begin{align*}
    \text{minimize} \;\;& B & \\
    \text{subject to} \;\;& \sum_{j=1}^n \alpha_j
    X_{k,b}(\q_i,\beta_j) \leq B
    &\forall\; i\in\{1,\ldots,m\},\\[-.4em]
    &&k\in\{1,\ldots,\ell(\q_i)\},\\
    & &b\in\{0,\ldots,\ell(\q_i)-1\}\\
    &\sum_{j=1}^n \alpha_j K(\q_i,\beta_j) > 0
    &\forall\; i\in\{1,\ldots,m\}\\
    & \frac{1}{\vmax(j)} \leq \alpha_j \leq \frac{1}{\vmin(j)},
    &\forall\; j\in\{1,\ldots,n\}.
  \end{align*}
  The optimization variables are $\alpha_j$ and $B$ and the quantities
  $X_{k,b}(\q_i,\beta_j)$ and $K(\q_i,\beta_j)$ are defined
  in~\eqref{eq:X_kb} and~\eqref{eq:C_def}.
\end{theorem}
From the above theorem, we can minimize the maximum value of the field
using a linear program.  This optimization has $n+1$ variables ($n$
basis function coefficients, and one upper bound $B$).  The number of
constraints is $m \sum_{i=1}^m\ell(\q_i)^2 + m + 2n$.  In practice,
$\ell(\q_i)$ is small compared to $n$ and $m$, and is independent of
$n$ and $m$.  Thus, for most instances, the linear program has
$O(2n+m)$ constraints.

\section{Multi-Robot Speed Controller}
\label{sec:multi_robot}
In this section we turn to the multi-robot case.  We find that a field
stabilizing controller can again be formulated as the solution of a
linear program.  Surprisingly, the resulting multi-robot controller
does not rely on direct communication between the robots.  We also
show that the optimal controller (the one that minimizes the steady
state field) for multiple robots cannot be formulated as an LP as in
the single robot case.  Finding the optimal multi-robot controller is
a subject of ongoing work.

The multiple robots travel on fixed paths, but those paths my be
different (or they may be the same) and they may intersect arbitrarily
with one another.  The robots may have different consumption rates,
footprints, and speed limits.  We do not explicitly deal with
collisions in this section, though the controller we propose can be
augmented with existing collision avoidance strategies.  This is the
subject of on going work.

We must first modify our notation to accommodate multiple robots.
Consider $N$ robots with closed paths $\gamma_r:[0,1] \to \real^2$, $r
\in\{1, \ldots,N\}$ where $\gamma_r$ and $\gamma_{r'}$ may intersect
with each other arbitrarily (e.g., they may be the same path, share
some segments, or be disjoint).  We again assume that the
parametrization of each curve is an arc-length parametrization,
normalized to unity.  Robot $r$, which traverses path $\gamma_r$ at
position $\theta_r(t)$, has a consumption rate $c_r(\q)$ over a
footprint $\BB_r\big(\theta_r(t)\big)$, and has a speed controller
$v_r(\theta_r)$ with maximum and minimum speed constraints
$v_{r,\min}(\theta_r)$ and $v_{r,\max}(\theta_r)$, respectively.  Let
$\RR_r = (\BB_r, v_{r,\min}, v_{r,\max})$ be the tuple containing the
parameters for robot $r$ and redefine $\RR := (\RR_1,\ldots,\RR_N)$ to
be the tuple containing all the robots' tuples of parameters.  Also,
the set of points $\theta_r$ from which $\q$ is in robot $r$'s
footprint is denoted $F_r(\q)$.  Furthermore, we let $\gamma :=
(\gamma_1, \ldots , \gamma_N)$ and $c := (c_1, \ldots , c_N)$ be the
tuple of all robot's paths and all robots' consumption functions,
respectively.  The persistent task with $N$ robots is now written
$(\RR,\gamma,Q,p,c)$ as before, and we seek speed controllers
$v_r(\theta_r)$ to keep $Z(\q,t)$ bounded everywhere, as in
Definition~\ref{def:stability}.

We make the assumption that when multiple robots' footprints are over
the same point $\q$, their consumption rates are additive.
Specifically, let $\NN_\q(t)$ be the set of robots whose footprints
are over the point $\q$ at time $t$,
\[
  \NN_\q(t) := \{r \mid \q \in \BB_r\big(\theta_r(t)\big)\}.
\]
Then the rate of change of the function $Z(\q,t)$ is given by
\begin{equation}
\label{eq:multirobot_Z_diffeq}
\dot Z(\q,t) = 
\begin{cases}
  \displaystyle p(\q)-\sum_{r\in\NN_\q(t)}c_r(\q), &  \text{if $Z(\q,t) > 0$}, \\
  \displaystyle \Big(p(\q)-\sum_{r\in\NN_{\q}(t)}c_r(\q)\Big)^+, & \text{if $Z(\q,t) = 0$}.
\end{cases}
\end{equation} 
We can reformulate a stability condition analogous to Lemma
\ref{lem:stab_cond} to suit the multi-robot setting, but we must first
establish the notion of a common period for the speed controllers of all
the robots.  Let $T_r = \int_0^1v_r^{-1}(\theta)d\theta$ be the period
of robot $r$, and let $\tau_r(\q) =
\int_{F_r(\q)}v_r^{-1}(\theta)d\theta$ be the time in that period that
$\q$ is in robot $r$'s footprint.  The existence of a common period
rests on the following technical assumption.
\begin{assumption}[Rational Periods]
  \label{ass:rational_periods}
  We assume that the periods $T_r$ are rational numbers, so that there
  exist integers $\num_r$ and $\den_r$ such that $T_r = \num_r/\den_r$.  
\end{assumption}
An immediate consequence of Assumption \ref{ass:rational_periods} is
that there exists a common period $T$ such that $T/T_r \in \nat$ for
all $r$.  That is, each controller executes a whole number of
cycles over the time interval $T$.  Specifically, letting $T =
\Pi_{r=1}^N\num_r$, we have $T/T_r = \den_r\Pi_{r'=1,r'\not =r}^N\num_{r'}$.
Now, we can state the necessary and sufficient conditions for a
field stabilizing multi-robot controller.

\begin{lemma}[Multi-Robot Stability Condition]
  \label{lem:multirobot_stab_cond}
  Given a multi-robot persistent task, the set of controllers
  $\theta_r \mapsto v_r(\theta_r)$, $r\in\{1,\ldots,N\}$ is
  field stabilizing if and only if
  \begin{equation}
   \label{eq:multirobot_stability}
  \sum_{r = 1}^N\frac{\tau_r(\q)}{T_r}c_r(\q) > p(\q)
\end{equation}
for every $\q\in Q$, where $T_r = \int_0^1v_r^{-1}(\theta)d\theta$ and
$\tau_r(\q) = \int_{F_r(\q)}v_r^{-1}(\theta)d\theta$.
\end{lemma}
The lemma states an intuitive extension of Lemma \ref{lem:stab_cond},
which is that the total consumption per cycle must exceed the total
production per cycle at each point $\q\in Q$.
\begin{proof}
  The proof closely follows the proof of Lemma \ref{lem:stab_cond}.
  Consider the change of $Z(\q,t)$ at a point $\q$ over any time
  interval $T$, where $T$ is a common period of all the rational
  periods $T_r$, so that $T/T_r \in \nat$ for all $r$.  By integrating
  (\ref{eq:multirobot_Z_diffeq}) we have that
\begin{align*}
  Z(\q,t+T)-Z(\q,t)  \ge Tp(\q) - \sum_{r=1}^N\int_t^T \mathbf{I}_r(\tau,\q)c_r(\q)d\tau\\
  = Tp(\q) -
  \sum_{r=\NN_{q}}\sum_{k=1}^{T/T_r}\int_{t+(k-1)T_r}^{t+kT_r}\mathbf{I}_r(\tau,\q)c_r(\q)d\tau\\
  = T\Big(p(\q) - \sum_{r=1}^Nc_r(\q)\frac{\tau_r(\q)}{T_r}\Big),
\end{align*}
where $\mathbf{I}_r(t,\q)$ takes the value $1$ when $\q$ is in the footprint of
robot $r$ and $0$ otherwise.  To simplify notation, define $C(\q) :=
\sum_{r=1}^Nc_r(\q)$ and $\bar{C}(\q) :=
\sum_{r=1}^Nc_r(\q)\tau_r(\q)/T_r$.

First we prove necessity.  In order to reach a contradiction, assume
that the condition in Lemma \ref{lem:multirobot_stab_cond} is false,
but the persistent task is stable.  Then $T\big(p(\q) - \bar{C}(\q)\big) \ge 0$
for some $\q$, implying $Z(\q,t+T) \ge Z(\q,t)$ for some $\q$ and for
all $t$, which contradicts stability.  In particular, for an initial
condition $Z(\q,0) > Z_{\max}$, $Z(\q,Tk)>Z_{\max}$ for all $k =
1,\ldots$.

Now we prove sufficiency.  If the condition is satisfied, there exists
some $\epsilon > 0$ such that $T\big(p(\q)-\bar{C}(\q)\big)=
-\epsilon$.  Suppose that at some time $t$ and some point $\q$,
$Z(\q,t) > T\big(C(\q) - p(\q)\big)$ (if no such time and point
exists, the persistent task is stable).  Then for all times in the interval
$\tau \in [t,t+T]$, $Z(\q,t+\tau)>0$, and by
(\ref{eq:multirobot_Z_diffeq}) we have that $Z(\q,t+T) - Z(\q,t) =
T\big(p(\q) - \bar{C}(\q)\big) = -\epsilon$.  Therefore, after
finitely many periods $T$, $Z(\q,t)$ will become less than
$T\big(C(\q)-p(\q)\big)$.  Now for a time $t$ and a point $\q$ such
that $Z(\q,t) < T\big(C(\q)-p(\q)\big)$, for all times $\tau \in
[t,t+T]$ we have that $Z(\q,t+\tau) \le Z(\q,t) + p(\q)T < TC(\q)$.
Therefore, once $Z(\q,t)$ falls below $T\big(C(\q) - p(\q)\big)$
(which will occur in finite time), it will never again exceed
$TC(\q)$.  Therefore the persistent task is stable with $Z_{\max} = \max_{\q
  \in Q}TC(\q)$.
\end{proof}

\begin{remark}[Justification of Rational Periods]
  \label{rem:justification_of_rational_periods}
  Assumption \ref{ass:rational_periods} is required only for the sake
  of simplifying the exposition.  The rational numbers are a dense
  subset of the real numbers, so for any $\epsilon>0$ we can find
  $\num_r$ and $\den_r$ such that $\num_r/\den_r \le T_r \le
  \num_r/\den_r +\epsilon$.  One could carry the $\epsilon$ through
  the analysis to prove our results in general.  
  \oprocend
\end{remark}
\subsection{Synthesis of Field Stabilizing Multi-Robot Controllers} 
A field stabilizing controller for the multi-robot case can again be
formulated as the solution of a linear program, provided that we
parametrize the controller using a finite number of parameters.  Our
parametrization will be somewhat different than for the single robot
case, however.  Because there are multiple robots, each with its own
period, we must normalize the speed controllers by their periods.  We
then represent the periods (actually, the inverse of the periods) as
separate parameters to be optimized.  In this way we maintain the
independent periods of the robots while still writing the optimization
over multiple controllers as a single set of linear constraints.

Define a normalized speed controller $\bar{v}_r(\theta_r) :=
T_rv_r(\theta_r)$, and the associated normalized coverage time
$\bar{\tau}_r(\q_i) := \int_{F_r(\q_i)}\bar{v}_r^{-1}(\theta) d\theta
= \tau_r(\q_i)/T_r$. We parametrize $\bar{v}_i^{-1}$ as
\[
\bar{v}_r^{-1}(\theta_r) =
\sum_{j=1}^{n_r}\alpha_{rj}\beta_{rj}(\theta_r),
\]
where $n_r$ is the number of basis functions for the $r$th robot, and
$\alpha_{rj}\in\real$ and $\beta_{rj}:[0,1] \to \real_{\ge 0}$ are
robot $r$'s $j$th parameter and basis function, respectively.  It is
useful to allow robots to have a different number of basis functions,
since they may have paths of different lengths with different speed
limits.  Assuming the basis functions are normalized with $\int_0^1
\beta_{rj}(\theta)d\theta = 1$ for all $r$ and $j$, then we can
enforce that $\bar v_r$ is a normalized speed controller by requiring
\[
\sum_{j=1}^{n_r} \alpha_{rj} = 1 \quad \forall\; r\in\{1,\ldots,N\}.
\] 

As before, we could use any number of different basis functions, but
here we specifically consider rectangular functions of the form
(\ref{eq:rect_basis}).  We also define the frequency for robot $r$ as
$f_r:=1/T_r$, and allow it to be a free parameter, so that
\begin{equation}
  \label{eq:multirobot_speed_param}
  v_r(\theta_r) =f_r\bar{v}_r(\theta_r) =
  \frac{f_r}{\sum_{j=1}^{n_r}\alpha_{rj}\beta_{rj}(\theta_r)}.
\end{equation}
From (\ref{eq:multirobot_stability}), for the set of controllers to be
field stabilizing, we require
\[
\sum_{r=1}^N c_r(\q_i) \sum_{j=1}^{n_r}\alpha_{rj}
\int_{F_r(\q_i)}\beta_{rj}(\theta)d\theta > p(\q_i)
\]
for all $r\in\{1,\ldots,N\}$ and $\q_i \in Q$.  Thus, defining
\begin{equation}
\label{eq:K_r_def}
K_r(\q_i,\beta_j) := c_r(\q_i) \int_{F_r(\q_i)}\beta_{rj}(\theta)d\theta,
\end{equation}
the stability constraints become
\[
\sum_{r=1}^N \sum_{j=1}^{n_r}\alpha_{rj}K_r(\q_i,\beta_j) > p(\q_i).
\]

To satisfy the speed constraints, we also require that
$v_{r,\max}^{-1}(\theta_r) \le v_r^{-1}(\theta_r) \le v_{r,\min}^{-1}(\theta_r)$, which
from (\ref{eq:multirobot_speed_param}) leads to
\[
   \frac{f_r}{v_{r,\max}(\theta_r)} \leq
   \sum_{j=1}^n\alpha_{rj}\beta_{rj}(\theta_r) \le   \frac{f_r}{v_{r,\min}(\theta_r)},
\]
for all $\theta_r \in [0,1]$.  For the rectangular basis functions in
(\ref{eq:rect_basis}), this specializes to $f_r/v_{r,\min}(j) \leq
\alpha_{rj} \leq f_r/v_{r,\min}(j)$ for all $r$ and $j$, where
\[
  v_{r,\max}(j) := \inf_{\theta\in [(j-1)/n_r,j/n_r)}v_{r,\max}(\theta_r)
\] and 
\[
   v_{r,\min}(j) := \sup_{\theta\in [(j-1)/n_r,j/n_r)}v_{r,\min}(\theta_r).
\]
This gives a linear set of constraints for stability, which
allows us to state the following theorem.
\begin{theorem}[Field Stabilizing Multi-Robot Controller]
  \label{thm:stab_multi}
  A persistent task is stabilizable by a set of multi-robot speed
  controllers $v_r(\theta_r)$, $r\in \{1,\ldots,N\}$, of the form
  (\ref{eq:multirobot_speed_param}) if and only if the following
  linear program is feasible:
\begin{align*}
  \text{minimize} \;\; & 0 \\
  \text{subject to} \;\; & \sum_{r=1}^N\sum_{j=1}^{n_r}
  \alpha_{rj}K_r(\q_i,\beta_j) > p(\q_i), \\
  &\qquad\qquad\qquad\qquad\qquad\qquad\quad \forall\;
  i\in\{1,\ldots,m\} \\
  &\sum_{j=1}^{n_r} \alpha_{rj} = 1 \qquad \qquad \qquad \qquad \forall\;
    r\in\{1,\ldots,N\} \\
  &f_r > 0 \qquad\qquad\qquad\qquad\qquad \; \forall\;
  r\in\{1,\ldots,N\} \\
 &\frac{f_r}{v_{r,\min}(j)} \leq \alpha_{rj} \leq \frac{f_r}{v_{r,\min}(j)}
\quad \forall \; j\in\{1,\ldots,n_r\}, \\
& \qquad\qquad\qquad\qquad\qquad\qquad\quad \;\;r\in\{1,\ldots,N\}, 
  \end{align*}
  where each $\alpha_{rj}$ and $f_r$ is an optimization variable, and
  $K_r(\q_i,\beta_j)$ is defined in~\eqref{eq:K_r_def}.
\end{theorem}

The above linear program has $\sum_{r=1}^N n_r + N$ variables (one for
each basis function coefficient $\alpha_{rj}$, and one for each
frequency $f_r$), and $m + \sum_{r=1}^N 2(n_r+1)$ constraints.  Thus, if we
use $n$ basis functions for each of the $N$ robot speed controllers, then
the number of variables is $(n+1) N$ and the number of constraints is $m+
2N(n+1)$.  Therefore, the size of the linear program grows linearly with
the number of robots.

\begin{remark}[Maximizing Stability Margin]
\label{rmk:MinCycleMultiController}
As summarized in Corollary~\ref{rmk:MinCycleController}, rather than
using the trivial cost function of $0$ in the LP in Theorem
\ref{thm:stab_multi}, one may wish to optimize for a meaningful
criterion.  For example, the controller that gives the minimum number
of common periods to steady-state can be obtained by maximizing $B$
subject to $\sum_{r=1}^N\sum_{j=1}^{n_r} \alpha_{rj}K_r(\q_i,\beta_j)
- p(\q_i) \ge B$ $\forall\; i\in\{1,\ldots,m\}$, in addition to the
other constraints of Theorem \ref{thm:stab_multi}. \oprocend
\end{remark}

\begin{remark}[Minimizing the Steady State Field]
\label{rmk:OptimalSpeedController}
The reader will note that we do not find the speed controller that
minimizes the steady state field for the multi-robot case, as was done
for the single robot case.  The reason is that the quantities called
$N_{k-b,k}$ in the single robot case would depend on the relative
positions of the robots in the multiple robot case.  These quantities
would have to be enumerated for all possible relative positions
between the multiple robots and there are infinite such relative
positions.  Thus the problem cannot be posed as an LP in the same way
as the single robot case.  We are currently investigating alternative
methods for finding the optimal multi-robot controller, such as convex
optimization. \oprocend
\end{remark}

\section{Simulations}
\label{sec:simulations}

In this section we present simulation results for the single robot and
multi-robot controllers.  The purpose of this section is threefold:
(i) to discuss details needed to implement the speed control
optimizations, (ii) to demonstrate how controllers can be computed for
both discrete and continuous fields, and (iii) to explore robustness
to modeling errors, parameter uncertainty, robot tracking errors, and
stochastic field evolution.

The optimization framework was implemented in
MATLAB$^{\tiny\textregistered}$, and the linear programs were solved
using the freely available SeDuMi (Self-Dual-Minimization) toolbox.
To give the readers some feel for the efficiency of the approach, we
report the time to solve each optimization on a laptop computer with a
$2.66\unit{GHz}$ dual core processor and $4\unit{GB}$ of RAM.  The
simulations are performed by discretizing time, and thus converting
the field evolution into a discrete-time evolution.  To perform the
optimization, we need a routine for computing the set $F(\q)$ for each
field point $\q\in Q$.  This is done as follows.  We initialize a set
$F(\q)$ for each point $\q$, and discretize the robot path into a
finite set $\{\theta_1,\ldots,\theta_n\}$.  For the rectangular basis,
this disretization is naturally given by the set of basis functions.
We iteratively place the robot footprint at each point $\theta_i$,
oriented with the desired robot heading at that point on the curve,
and then add $\theta_i$ to each set $F(\q)$ for which $\q$ is covered
by the footprint.  By approximating the robot footprint with a
polygon, we can determine if a point lies in the footprint efficiently
(this is a standard problem in computational geometry).

Figure~\ref{fig:vehicle_2D_sim} shows a simulation for one ground
robot performing a persistent monitoring task of $10$ points (i.e.,
$|Q| =10$).  The environment is a $70\unit{m}$ by $70\unit{m}$ square,
and the closed path has a length of $300\unit{m}$.  For all points we
set the consumption rate $c(\q) = 1$ (in units
of$\unit{(field~height)/s}$).  For each yellow point $\q$ we set
$p(\q) = 0.15$, and for the single red point we set $p(\q) = 0.35$.
The robot has a circular footprint with a radius of $12\unit{m}$, and
for all $\theta$ the robot has a minimum speed of $\vmin
=0.2\unit{m/s}$ and a maximum speed of $\vmax = 2\unit{m/s}$.  If the
robot were to simply move at constant speed along the path, then $8$
of the $10$ field points would be unstable.  The speed controller
resulting from the optimization in Section~\ref{sec:optimal_speed} is
shown in Figure~\ref{fig:speed_profile_2}.  A total of $150$
rectangular basis functions were used. The optimization was solved in
less than $1/10$ of a second.  Using the speed controller, the cycle
time was $T = 420\unit{s}$.

\begin{figure}
  \centering
  \includegraphics[width=0.325\linewidth]{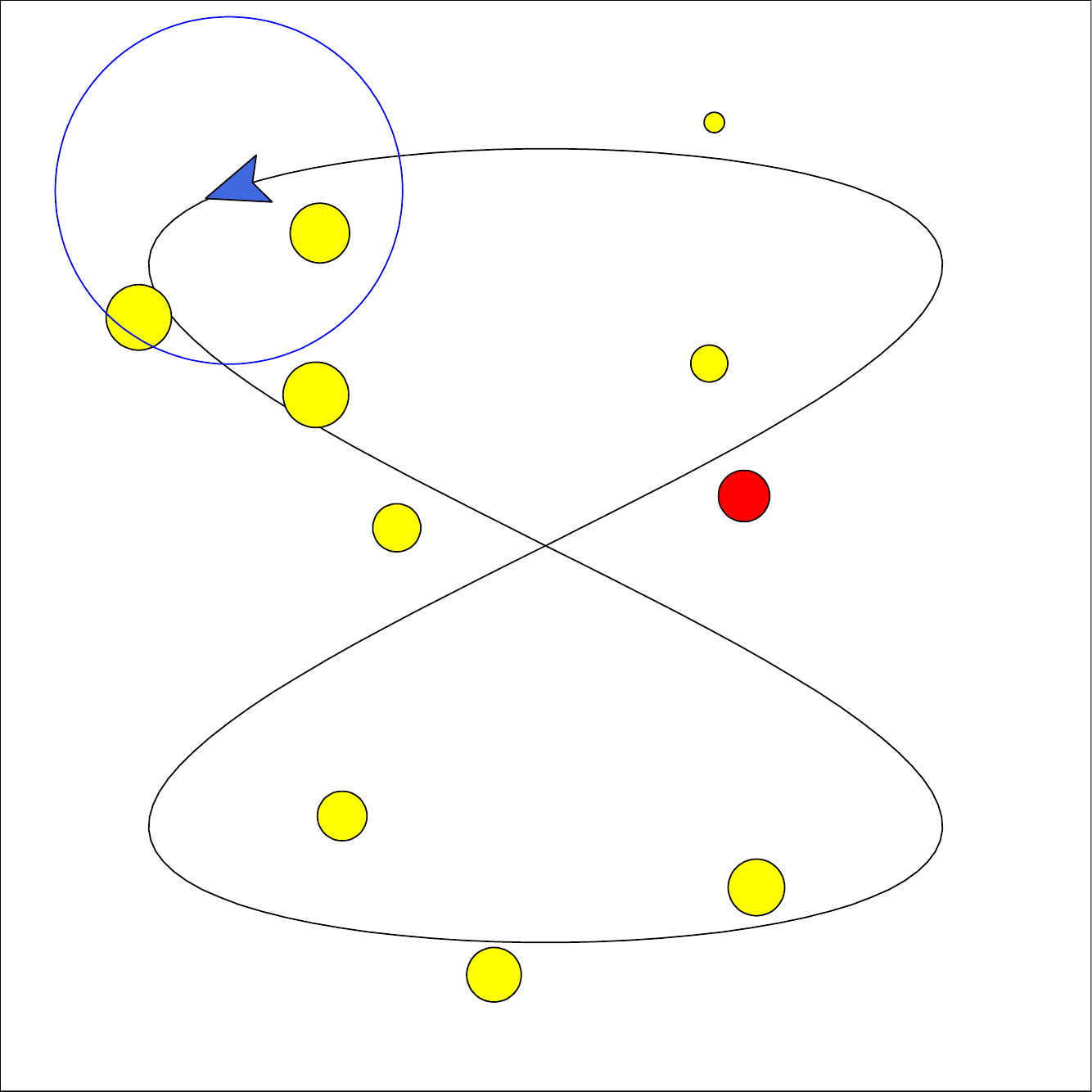} %
  \hfill
  \includegraphics[width=0.325\linewidth]{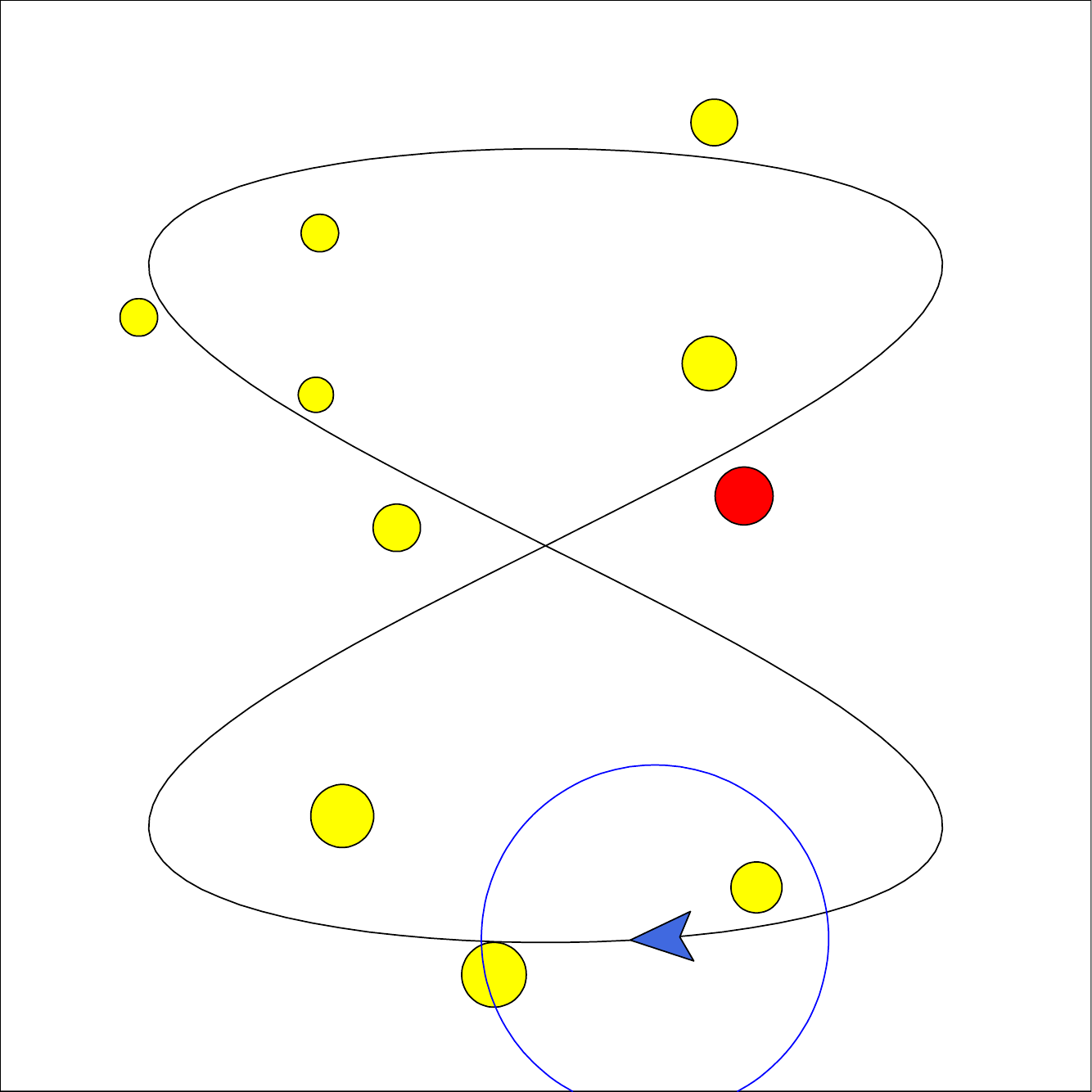} %
  \hfill
  \includegraphics[width=0.325\linewidth]{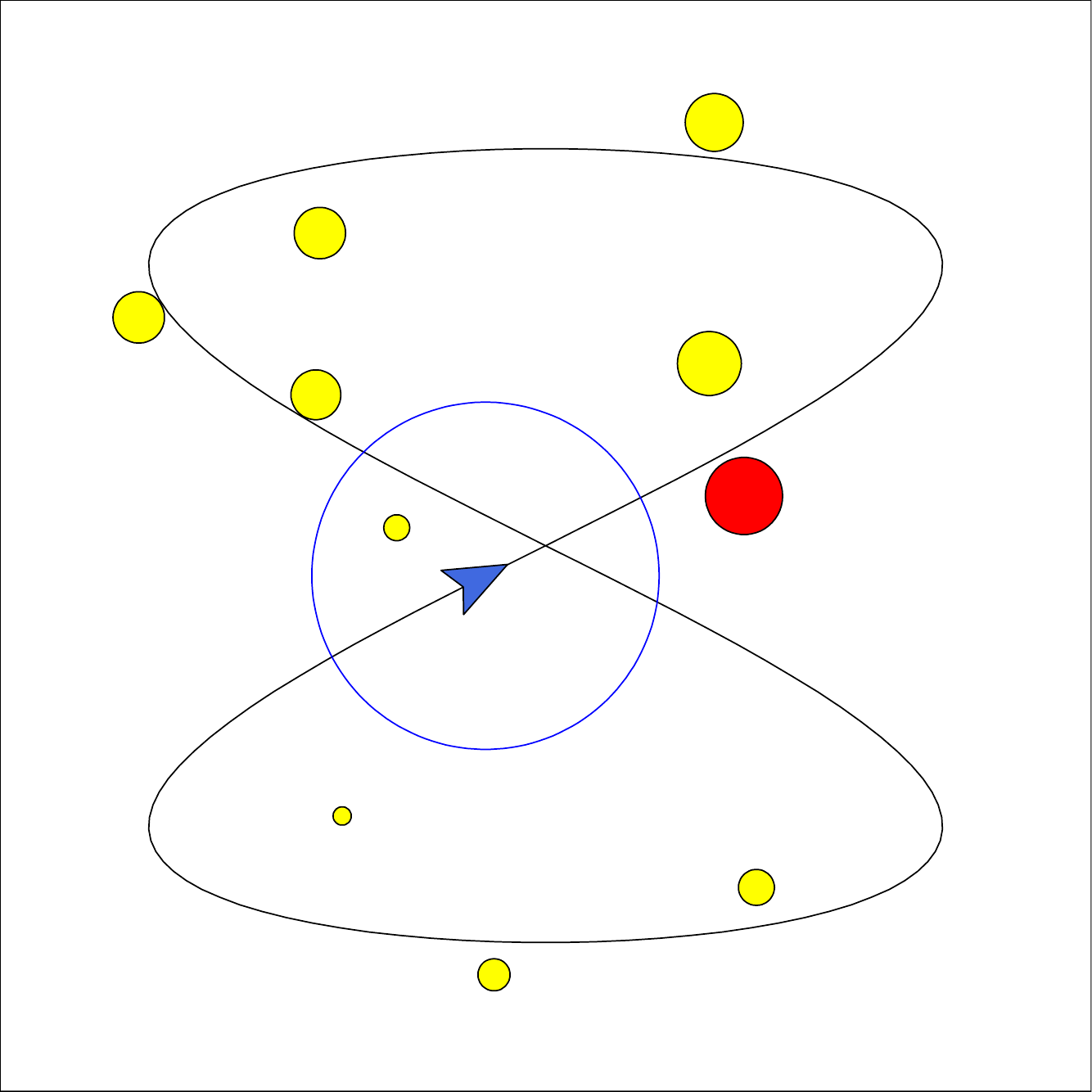}
  \caption{An example with a field consisting of $10$ points.  The
    field $Z(\q,t)$ at each point is indicated by the area of the disk
    centered at the point.  The vehicle footprint is a disk.  The time
    sequence of the three snapshots goes from left to right.  The
    vehicle is moving counter-clockwise around the top half of the
    figure eight and clockwise around the bottom half.  The time
    evolution of the red field point is shown in
    Figure~\ref{fig:2D_sim_speed_steady_state}.}
  \label{fig:vehicle_2D_sim}
\end{figure}
\begin{figure}
  \centering %
 \subfloat[The optimal speed controller.]{%
 \includegraphics[width=0.48\linewidth]{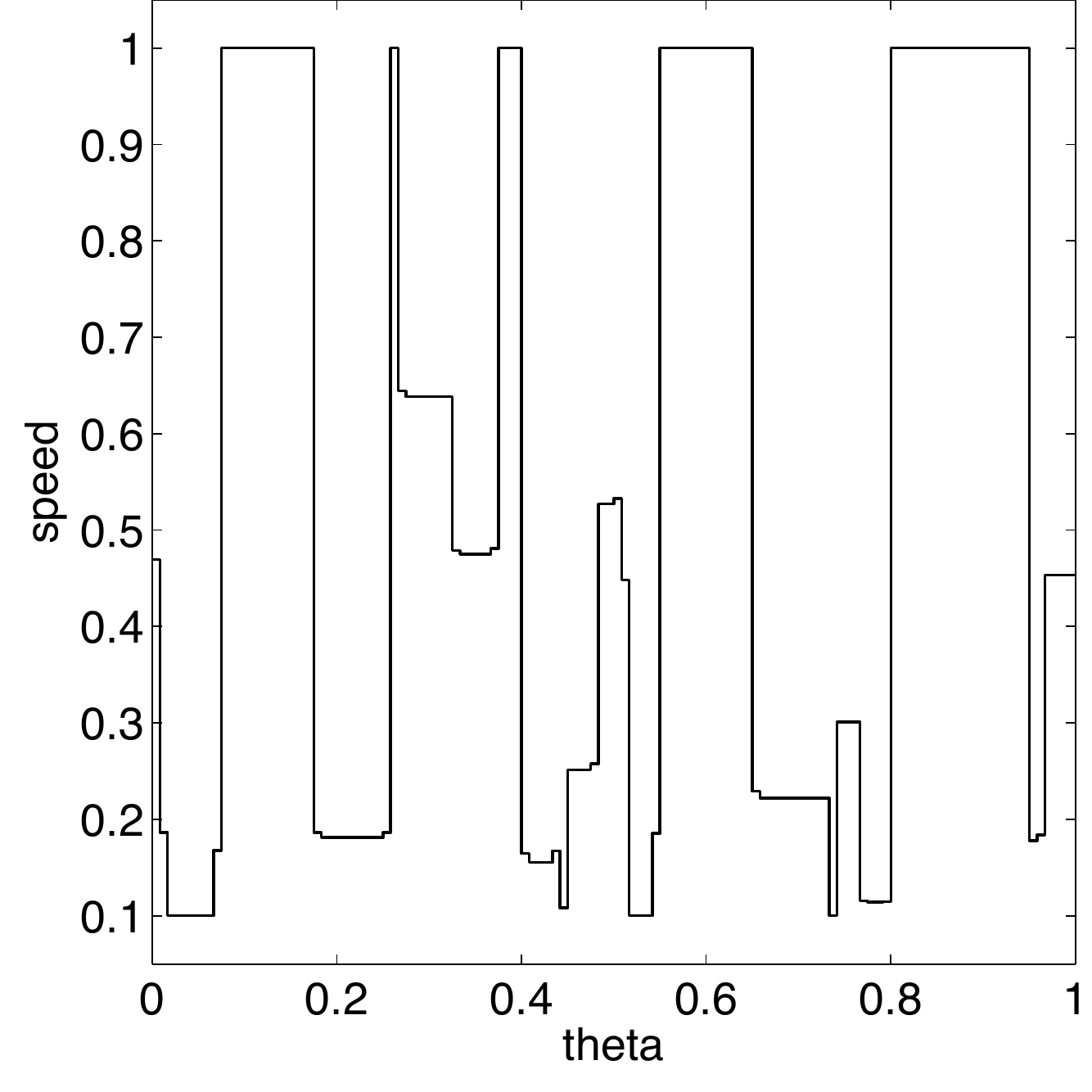}
 \label{fig:speed_profile_2}
 }%
 \hfill
  \subfloat[The field $Z(\q,t)$ at the red (dark) point.]{%
  \includegraphics[width=0.49\linewidth]{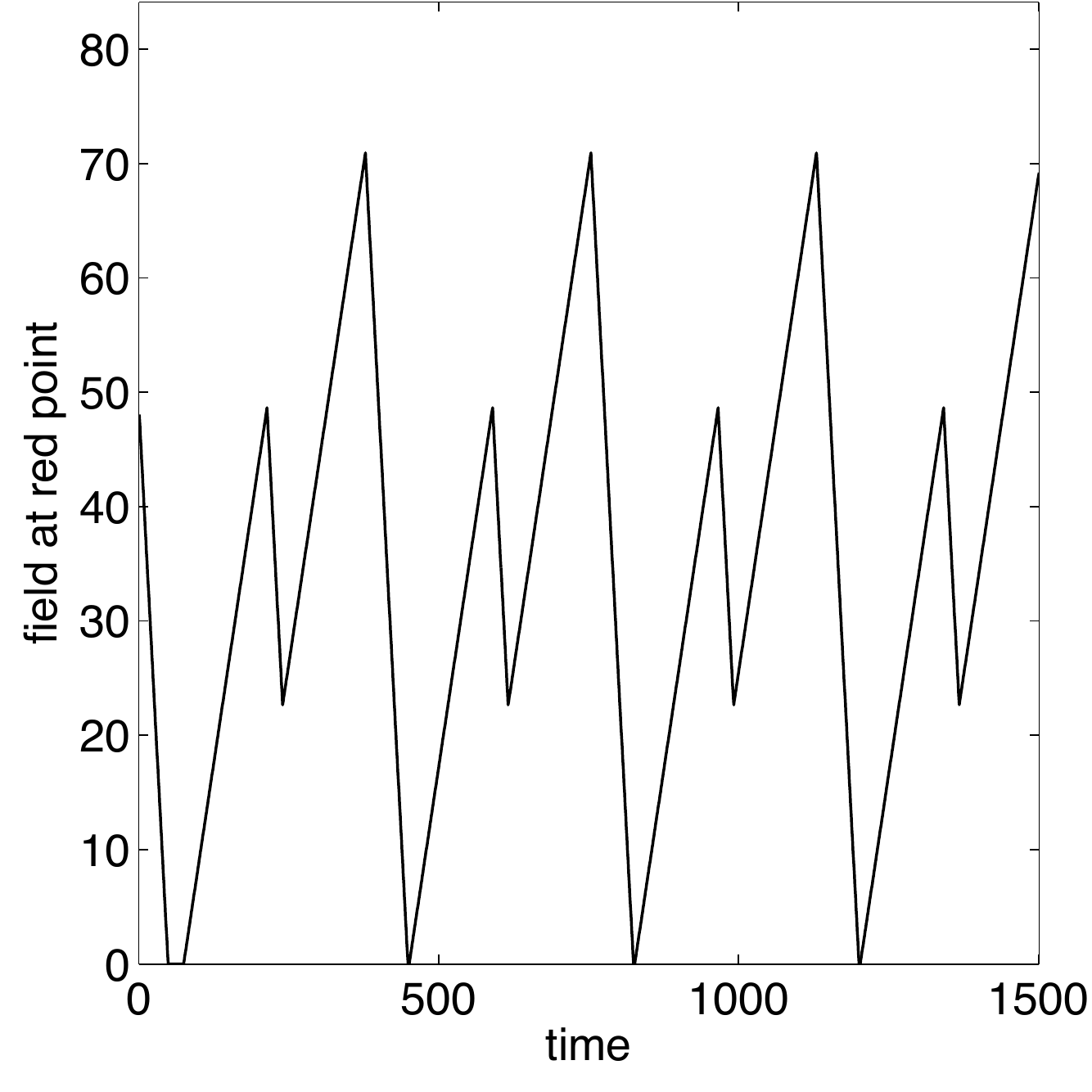}
  \label{fig:steady_state_2}
}%
\caption{The optimal speed controller corresponding to the curve and
  field shown in Figure~\ref{fig:vehicle_2D_sim}.  The minimum speed
  is $\vmin = 0.22\unit{m/s}$ and the maximum speed is $\vmax =
  2\unit{m/s}$.  The field $Z(\q,t)$ is shown for the red (dark) point
  in Figure~\ref{fig:vehicle_2D_sim}. It converges in finite time to a
  steady state.}
  \label{fig:2D_sim_speed_steady_state}
\end{figure}
The field $Z(\q,t)$ for the red (shaded) point in
Figure~\ref{fig:vehicle_2D_sim}, is shown as a function of time in
Figure~\ref{fig:steady_state_2}.  One can see that the field converges
in finite time to a periodic cycle.  In addition, the field goes to
zero during each cycle.  The periodic cycle is the steady-state as
characterized in Section~\ref{sec:optimal_speed}.

\begin{figure*}
  \centering
  \includegraphics[width=0.32\linewidth]{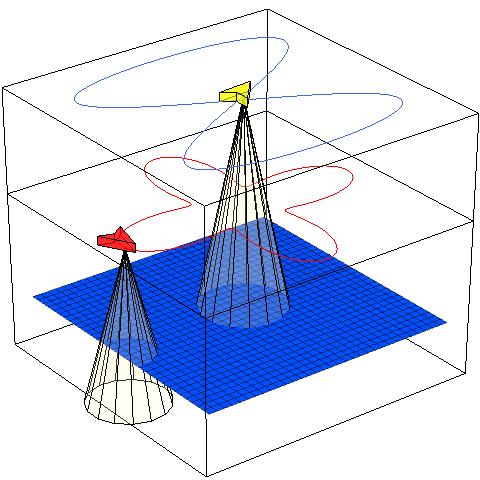} %
  \hfill
  \includegraphics[width=0.32\linewidth]{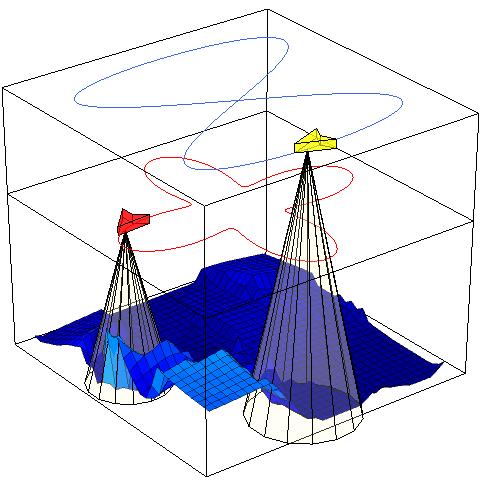} %
  \hfill
  \includegraphics[width=0.32\linewidth]{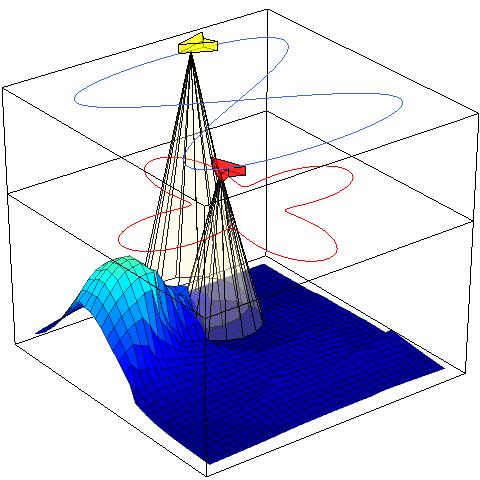}
  \caption{An example of using a discretized approximation of a
    continuous field for two robots.  The grid is $32 \times 32$, and
    the field $Z(\q,t)$ at each point is shown as the surface.  The
    footprint of the higher robot (yellow) is $4/3$ the radius of the
    footprint of the lower (red) robot.  The snapshot sequence goes
    from left to right with the left snapshot showing the initial
    condition of the robots and field.}
  \label{fig:multi_3D_sim}
\end{figure*}
In Figure~\ref{fig:multi_3D_sim}, a simulation is shown for the case
when the entire continuous environment must be monitored by two aerial
robots.  For multiple robots we synthesized a field stabilizing
controller by maximizing the stability margin, as discussed in
Remark~\ref{rmk:MinCycleMultiController}.  The continuous field is
defined over a $690\unit{m}$ by $690\unit{m}$ square, and was
approximated using a $32 \times 32$ grid. For all points $\q\in Q$ we
set the consumption rate $c(\q) = 1$, and the production function is
shown in Figure~\ref{fig:prod_fn}.
\begin{figure}
\centering
  \includegraphics[width=0.6\linewidth]{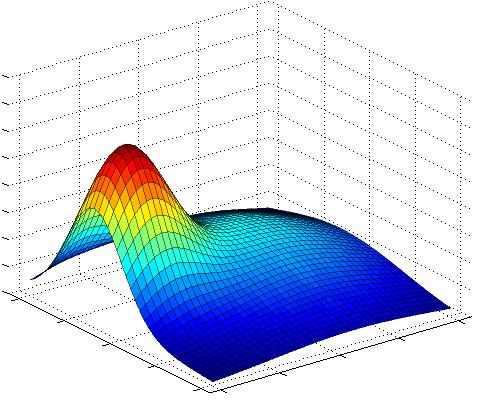}
  \caption{The production function $p(\q)$ for the simulations in
    Figures~\ref{fig:vehicle_3D_sim} and~\ref{fig:multi_3D_sim}.}
  \label{fig:prod_fn}
\end{figure}
Robot $1$ followed a figure-eight path which has a length of
$2630\unit{m}$, while robot $2$ followed a four-leaf clover path with
a length of $2250\unit{m}$.  The footprint for robot $1$, the higher
(yellow) robot had a radius of $100\unit{m}$, and the speed
constraints were given by $v_{\min,1} = 1.5\unit{m/s}$ and $v_{\max,1}
= 15\unit{m/s}$.  The footprint for robot $2$, the lower (red) robot,
had a radius of $133\unit{m}$, and the speed constraints were given by
$v_{\min,2} = 2\unit{m/s}$ and $v_{\max,2} = 20\unit{m/s}$.  The cycle
time for robot $1$ was $T_1 = 519\unit{s}$, and the cycle time for
robot $2$ was $T_2=443\unit{s}$.  A total of $150$ rectangular basis
functions were used for each robot's speed controller.  The
optimization was solved in approximately $10$ seconds.  A snapshot for
a simulation with three robots is shown in Figure~\ref{fig:sample_PM}.
In this case, the green robot flying at a higher altitude has a square
footprint which is oriented with the robot's current heading.

\subsection{A Case Study in Robustness}

In this subsection we demonstrate how we can compute a speed
controller that exhibits robustness to motion errors, modeling
uncertainty, or stochastic fluctuations.  This is important since the
speed controller is computed offline.  It should be noted, however,
that in practice the controller can be recomputed periodically during
the robot's execution, taking into account new information about the
field evolution.

As shown in Corollary~\ref{rmk:MinCycleController} we can maximize
stability margin of a speed controller.  The corollary showed that in
maximizing this metric we obtain some robustness to error.  To explore
the robustness properties of this controller, consider the single
robot example in Figure~\ref{fig:vehicle_3D_sim}.  In this example,
the square $665\unit{m}$ by $665\unit{m}$ environment must be
monitored by an aerial robot. We approximated the continuous field
using a $32 \times 32$ grid.  The consumption rate was set to $c(\q) =
5$, for each field point $\q\in Q$.  The production rate of the field
was given by the function shown in Figure~\ref{fig:prod_fn}.  The
maximum production rate of the field was $0.74$ and the average was
$0.21$.  The robot had a circular footprint with a radius of
$133\unit{m}$, a minimum speed of $\vmin = 1.5\unit{m/s}$ and a
maximum speed of $\vmax = 15\unit{m/s}$.  The path had a cycle length
of $4200\unit{m}$.  If the robot followed the path at constant speed,
then $80$ of the points would be unstable.

For the speed controller, we used $280$ rectangular basis functions,
and solved the optimization as described in
Corollary~\ref{rmk:MinCycleController}, resulting in a stability
margin of $B = 97.8$.  The optimization was solved in approximately
$10$ seconds.  The time for the robot to complete one cycle of the
path using this controller was $T = 439\unit{s}$.
\begin{figure*}
  \centering
  \includegraphics[width=0.32\linewidth]{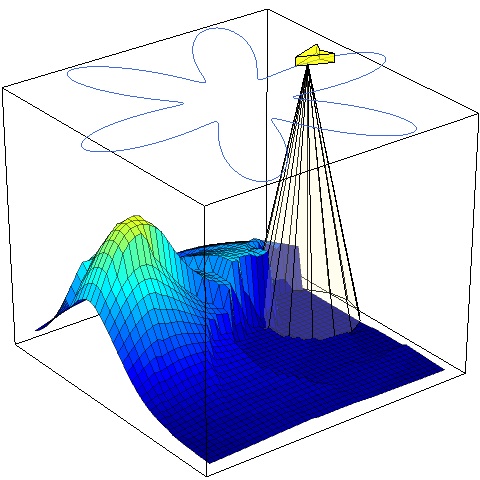} %
  \hfill
  \includegraphics[width=0.32\linewidth]{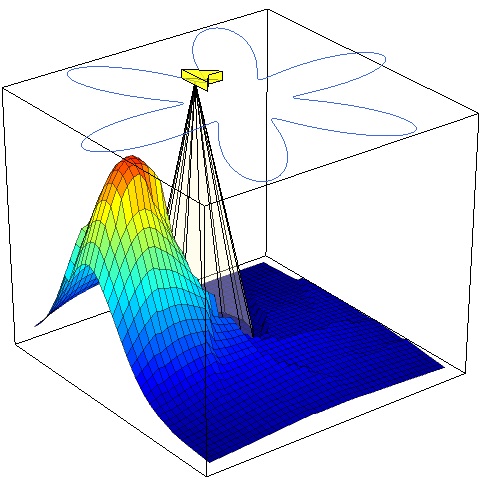} %
  \hfill
  \includegraphics[width=0.32\linewidth]{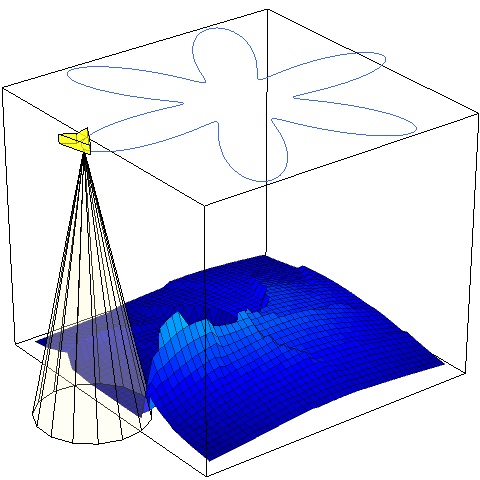}
  \caption{An example of using a discretized approximation of a
    continuous field.  The field $Z(\q,t)$ at each point is shown as
    the surface.  The footprint of the vehicle is a disk, and the
    robot's trajectory is given by the ``six leaf'' clover.  The time
    sequence of the three snapshots goes from left to right.  In the
    rightmost snapshot the vehicle has just finished reducing the
    large peak that forms in the left and center snapshots.}
  \label{fig:vehicle_3D_sim}
\end{figure*}

\textbf{Stochastic field evolution:} Now, suppose that we add
zero-mean noise to the production function.  Thus, the robot based its
speed controller on the ``nominal'' production function $\bar p(\q)$
(shown in Figure~\ref{fig:prod_fn}), but the actual production
function is $p(\q,t) = \bar p(\q) + n(t,\q)$, where $n(t,\q)$ is
noise.  For the simulation, at each time instant $t$, and for each
point $\q\in Q$, we independently draw $n(t,\q)$ uniformly from the
set $[-n_{\max},n_{\max}]$, where $n_{\max} > 0$ is the maximum value
of the noise.

The simulations were carried out as follows.  We varied the magnitude
of the noise $n_{\max}$, and studied the maximum value reached by the
field.  For each value $n_{\max}$, we performed $20$ trials, and in
each trial we recorded the maximum value reached by the field on a
time horizon of $2500\unit{s}$.  In Figure~\ref{fig:noise_exp}, we
display statistics from the $20$ independent trials at each noise
level, namely the mean, minimum, and maximum, as well as the standard
deviation.  With zero added noise, one can see that the mean, minimum,
and maximum all coincide.  As noise is added to the evolution, the
difference between the minimum and maximum value grows.  However, it
is interesting to note that while the performance degrades (i.e., the
mean increases with increasing noise), the system remains stable.
Thus, the simulation demonstrates some of the robustness properties of
the proposed controller.

\begin{figure}
\centering
\includegraphics[width=0.7\linewidth]{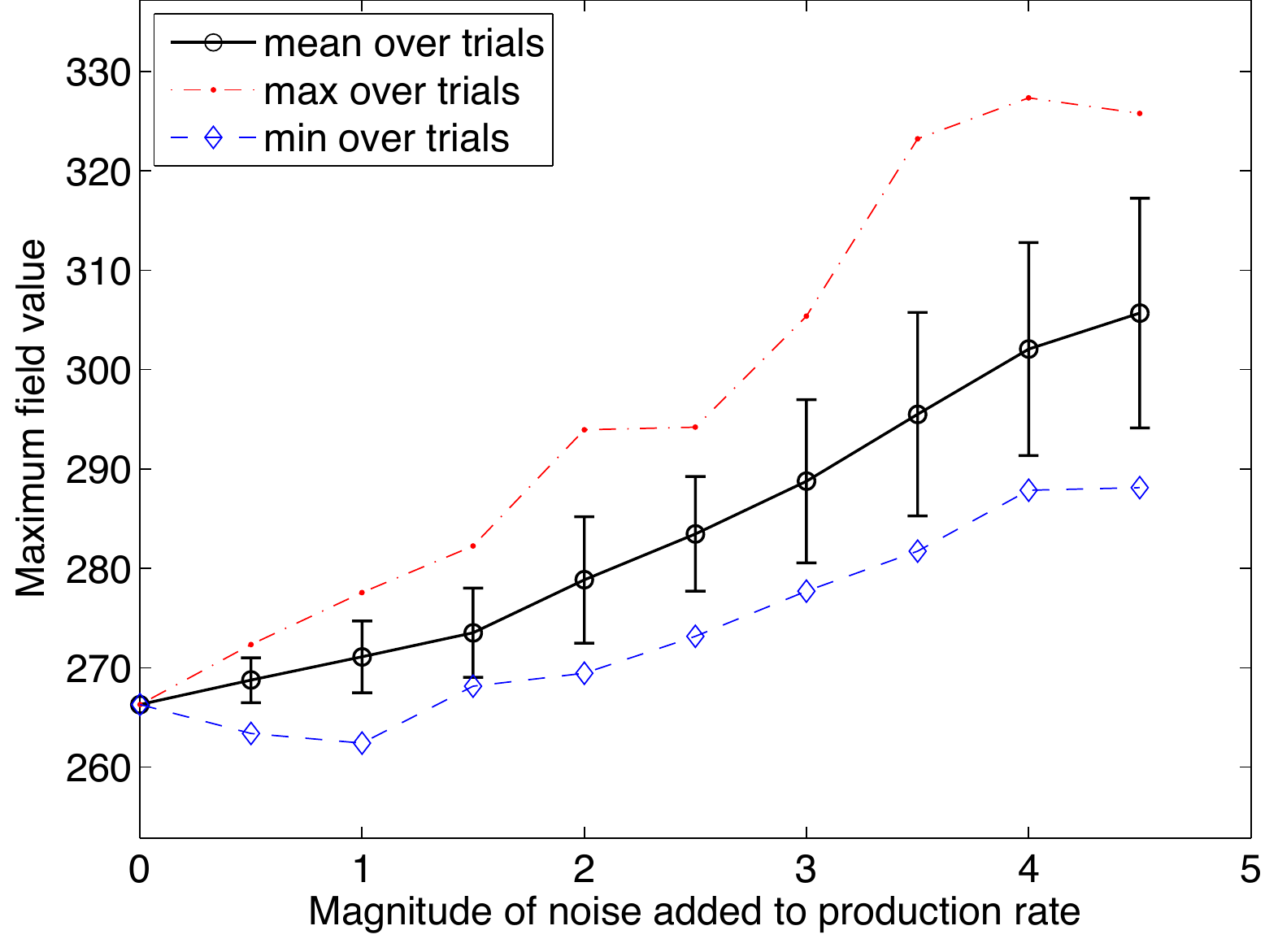}
\caption{The robustness of the speed controller to noise in the
  production rate.  For each value of noise $n_{\max}$, the plot shows
  statistics on the maximum value reached by the field over 20
  independent trials.}
\label{fig:noise_exp}
\end{figure}

\textbf{Parameter errors:} We can also consider robustness to
parameter errors. In particular, consider the case where the robot
bases its optimization on a production rate of $\bar p(\q)$ (shown in
Figure~\ref{fig:prod_fn}), but the actual production rate is given by
$p(\q) = \bar p(\q) +\epsilon$, where $\epsilon >0$ (note that the
field is trivially stable for any $\epsilon \leq 0$). By maximizing
the stability margin, we obtain some level of robustness against this
type of parameter uncertainty.  The amount of error that we can
tolerate is directly related to the stability margin $B = 97.8$, as
shown in Corollary~\ref{rmk:MinCycleController}. In particular, we
obtain $\epsilon < B c(\q_i) /\big(\sum_{j=1}^n \alpha_j \int_0^1
\beta_j(\theta)d\theta\big) = 0.074$. We performed simulations of
monitoring task for successively larger values of $\epsilon$.  From
this data, we verified that the field remains stable for any $\epsilon
\leq 0.07$.  For this example, the average value of $p(\q)$ over all
points $\q$ is $0.21$, and so we can handle uncertainty in the
magnitude of the production rate on the order of $30\%$.

\textbf{Tracking error:} After running the speed optimization, a robot
has a desired trajectory, consisting of the pre-specified path and the
optimized speed along the path.  In practice, this trajectory will be
the input to a tracking controller, which takes into account the robot
dynamics.  Since there are inevitably tracking errors, the stability
margin of the controller is needed to ensure field stability.  As an
example, we considered a unicycle model for an aerial robot. In this
model, the robot's configuration is given by a heading $\phi$ and a
position $(x,y)$.  The control inputs are the linear and angular
speeds: $\dot x = v\cos \phi$, $\dot y = v\sin\phi$, and $\dot \phi =
\omega$. The linear speed had bounds of $\vmin = 1.5\unit{m/s}$ and
$\vmax = 15\unit{m/s}$, and the angular speed was upper bounded by
$0.5\unit{rad/s}$.  We used the same speed controller as in the
previous two examples (maximizing the stability margin).  For
trajectory tracking, we used a dynamic feedback linearization
controller~\cite{GO-ADL-MV:02}.  We chose conservative controller
gains of $0.5$ for the proportional and derivative control in order to
accentuate the tracking error.  The results are shown in
Figure~\ref{fig:tracking}.  Due to the stability margin of $97.8$, the
field remains stable, even in the presence of this tracking error.
However, in simulation, the maximum field height increased by about
$13\%$ from $268$ (as shown for the zero noise case in
Figure~\ref{fig:noise_exp}) to $305$.

\begin{figure}
\centering
\includegraphics[width=0.48\linewidth]{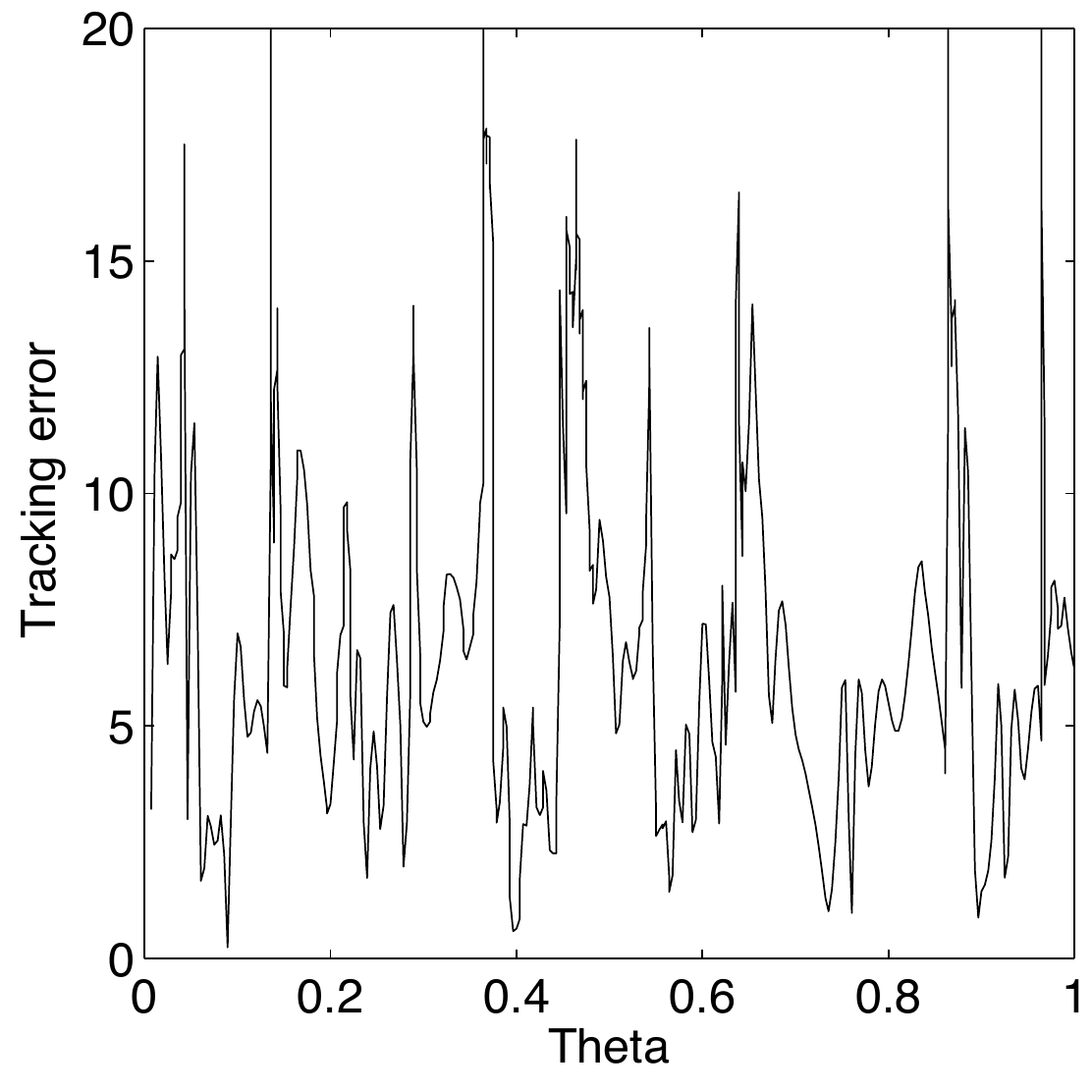} \hfill
\includegraphics[width=0.48\linewidth]{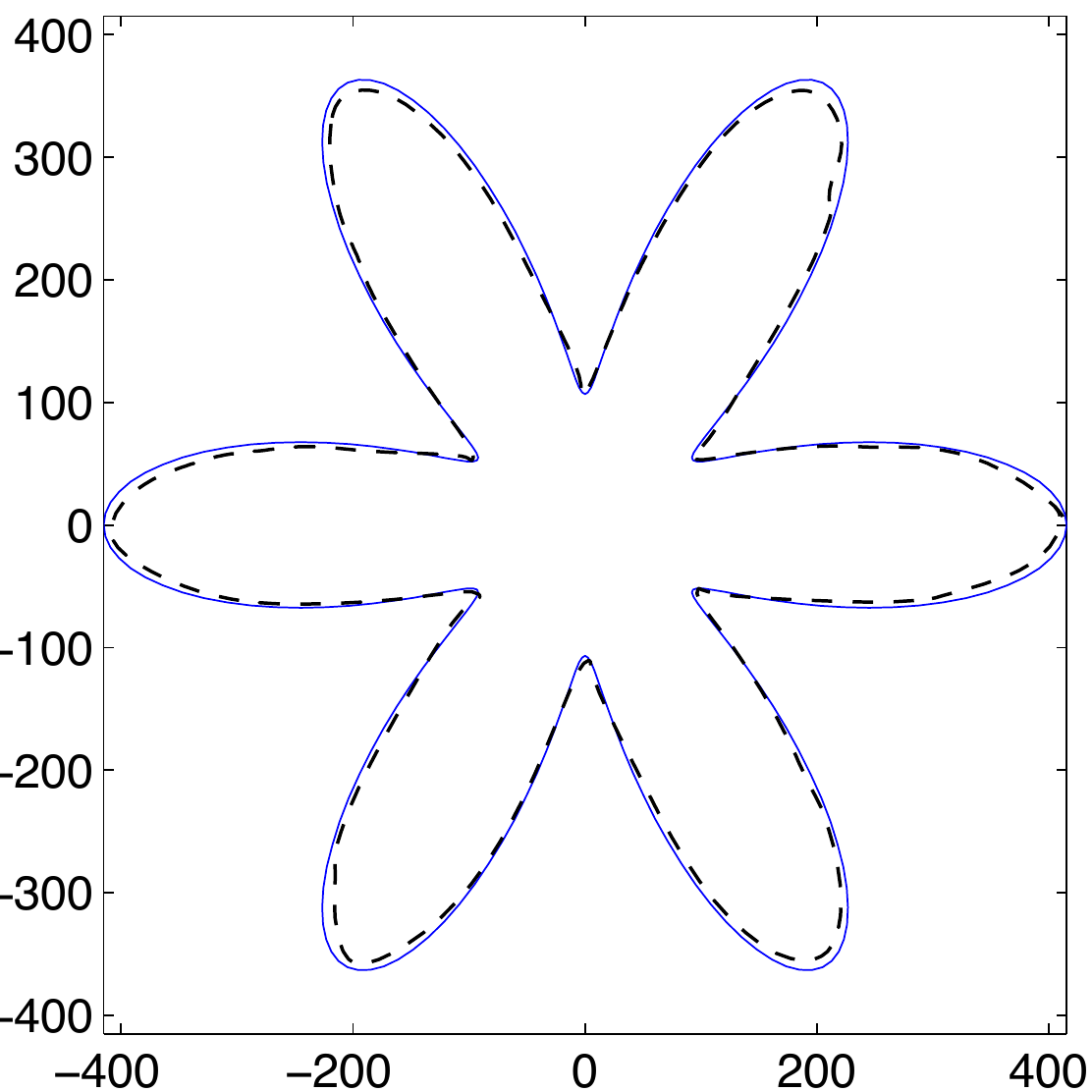}
\caption{The robustness of the speed controller to tracking errors.
  The left figure shows the absolute tracking error in meters.  In the
  right figure, the solid line is the desired path, and the dashed
  line is the path executed by the unicycle robot.}
\label{fig:tracking}
\end{figure}

\section{Conclusions, Extensions, and Open Problems}
\label{sec:conclusions}
In this paper we proposed a model for persistent sweeping and
monitoring tasks and derived controllers for robots to accomplish
those tasks.  We specifically considered the case in which robots are
confined to pre-specified, closed paths, along which their speed must
be controlled.  We formulated an LP whose solution gives speed
controllers that keep the accumulation bounded everywhere in the
environment for single robots and multiple robots.  For single robots,
we also formulated a different LP to give the optimal controller---the
one that keeps the accumulation function as low as possible
everywhere.  We see this as a solution to one kind of persistent task
for robots, but many open problems remain.

\subsection{Extensions and Open Problems}
\label{sec:defining_robotic_persistence}
We are interested in the general problem of solving persistent tasks,
which we broadly define as tasks that require perpetual attention by a
robotic agent.  The main objective of a persistent task is to maintain
the accumulation of some undesirable quantity at as low a value as
possible over an environment using one or multiple robotic agents.
The difficulty of this problem depends on what is known by the robots,
and precisely what the robots' capabilities are.  Let us enumerate
several possible dimensions for extension on this problem.
\begin{itemize}
\item \textbf{Trajectory vs. Path vs. Speed} One might consider
  controlling only the speed over a prescribed path, as we have done
  in this paper, only the path with a prescribed speed, or complete
  trajectory planning.

\item \textbf{Single vs. Multiple Robots} There may be only one robot,
  or there may be a team of robots, potentially with heterogeneous
  capabilities and constraints.

\item \textbf{Known vs. Unknown Production Rate} The robot may know (or
  be able to sense) the production rate, or it may have to learn it
  over time from measurements.

\item \textbf{Constant vs. Time-Varying Production Rate} The production
  rate may be constant in time, or it may change indicating a change
  in the environment's dynamics.

\item \textbf{Finite vs. Continuum Points of Interest} The points of
  interest in the environment may be viewed as a finite set of
  discrete points over which to control the accumulation, or as an
  infinite continuum in which we would like to control the
  accumulation at every point.
\end{itemize}
In this paper we specifically considered \emph{speed} control over a
prescribed path of both \emph{single and multiple} robots in a
\emph{finite} environment with a \emph{known}, \emph{constant}
production rate.

One interesting direction in which to expand this work is to consider
planning full trajectories for robots to carry out persistent tasks.
The high dimensionality of the space of possible trajectories makes
this a difficult problem.  However, if the robot's path is limited by
some inherent constraints, then this problem may admit solutions with
guarantees.  For example, underwater gliders are commonly constrained
to take piecewise linear paths, which can be parametrized with a low
number of parameters.  Another direction of extension is to have a
robot solve the LP for its controller on-line.  This would be useful
if, for example the production rate is not known before hand, but can
be sensed over the sensor footprint of the robot.  Likewise if the
production rate changes in time, it would be useful for a robot to be
able to adjust its controller on-line to accommodate these changes.  A
promising approach for this is to repeatedly solve for the LP in a
receding horizon fashion, using newly acquired information to update
the model of the field evolution.  We continue to study problems of
persistent tasks for robots in these and other directions.

\appendix[Periodic Position-Feedback Controllers]
\label{sec:appendix}

In this appendix we will prove Proposition~\ref{prop:periodic_cont} on
periodic position-feedback controllers.

Consider a general speed controller $(Z,\mathrm{IC},t)\mapsto
v(\theta,Z,\mathrm{IC},t)$ where $\mathrm{IC} := (Z(\q,0),\theta(0))$ is
the set of initial conditions.  Since $\vmin(\theta) >0$ for all
$\theta\in[0,1]$, the value of $\theta$ strictly monotonically
increases from $0$ to $1$ for every valid controller (once it reaches
$1$ it then wraps around to $0$).  In addition, the evolution of $Z$
is deterministic and is uniquely determined by the initial conditions
and the speed controller, as given in~\eqref{eq:Z_diffeq}.  Because of
this, every controller of the form $v(\theta,Z,\mathrm{IC},t)$ can be
written as an infinite sequence of controllers
$v_1(\theta,\mathrm{IC}),v_2(\theta,\mathrm{IC}),\ldots$, where
controller $v_k(\theta,\mathrm{IC})$ is the controller used during the
$k$th period (or cycle).

With the above discussion in place, we are now ready to prove
Proposition~\ref{prop:periodic_cont}

\begin{proof}[Proof of Proposition~\ref{prop:periodic_cont}]
  To begin, consider a feasible persistent monitoring task and a
  field stabilizing controller of the form $v(\theta,Z,\mathrm{IC},t)$, where
  $\mathrm{IC} := (Z(\q,0),\theta(0))$.  Without loss of generality,
  we can assume that $\theta(0) := 0$.  From the discussion above, we
  can write the general controller as a sequence of controllers
  $v_1(\theta,\mathrm{IC}),v_2(\theta,\mathrm{IC}),\ldots$, where
  controller $v_k(\theta,\mathrm{IC})$, $k\in\nat$ is used on the
  $k$th period (or cycle).

  Since the controller is stable, there exists a $Z_{\max}$ such that
  for every set of initial conditions $\mathrm{IC}$, we have
  $\limsup_{t\to+\infty} Z(\q,t) \leq Z_{\max}$.  Let us fix $\epsilon
  > 0$ and fix the initial conditions to a set of values
  $\overline{\mathrm{IC}}$ such that $Z(\q,0) > Z_{\max} +\epsilon$
  for all $\q\in Q$. Now, we will prove the result by constructing a
  periodic position-feedback controller.
  
  Let $t_1 = 0$ and define
  \[
  t_k = t_{k-1} + \int_0^1 \frac{1}{v_k(\theta,\overline{\mathrm{IC}})}d\theta, 
  \]
  for each integer $k \geq 2$.  Thus, controller $v_k$ is used during
  the time interval $[t_k,t_{k+1}]$.  Following the same argument as
  in the proof of Lemma~\ref{lem:stab_cond}, we have that for each
  $\q\in Q$,
  \begin{multline*}
    Z(\q,t_k) - Z(\q,t_1) \geq p(\q) \sum_{\ell = 1}^k \int_0^1
    \frac{1}{v_{\ell}(\theta,\overline{\mathrm{IC}})}d\theta \\- c(\q) \sum_{\ell=1}^k
    \int_{F(\q)} \frac{1}{v_{\ell}(\theta,\overline{\mathrm{IC}})}d\theta.
  \end{multline*}
  
  Now, $Z(\q,t_1) > Z_{\max} +\epsilon$ is the initial condition, and
  $\limsup_{t\to+\infty} Z(\q,t) \leq Z_{\max}$.  Thus, for every
  fixed $\delta \in(0, \epsilon)$, there exists a finite $k$ such that
  $Z(\q,t_k) \leq Z_{\max} + \delta$.  Since this must hold for every
  $\q\in Q$, we see that there exists an integer $k$ such that
  \[
  p(\q) \sum_{\ell = 1}^k \int_0^1 \frac{1}{v_{\ell}(\theta,\overline{\mathrm{IC}})}d\theta -
  c(\q) \sum_{\ell=1}^k \int_{F(\q)}
  \frac{1}{v_{\ell}(\theta,\overline{\mathrm{IC}})}d\theta < 0,
  \]
  for every $\q\in Q$.  Rearranging the previous equation we obtain
  \begin{equation}
    \label{eq:period_depend_stab}
    c(\q)\int_{F(\q)} \sum_{\ell=1}^k \frac{1}{v_{\ell}(\theta,\overline{\mathrm{IC}})}d\theta >
    p(\q)\int_0^1 \sum_{\ell = 1}^k
    \frac{1}{v_{\ell}(\theta,\overline{\mathrm{IC}})}d\theta.
  \end{equation}
  
  Therefore, let us define the periodic controller
  \[
  v(\theta) = k \left(\sum_{\ell = 1}^k
    \frac{1}{v_{\ell}(\theta,\overline{\mathrm{IC}})} \right)^{-1},
  \]
  for each $\theta\in[0,1]$.  Note that if $\vmin(\theta) \leq
  v_{\ell}(\theta,\overline{\mathrm{IC}}) \leq \vmax(\theta)$ for all
  $\theta\in[0,1]$ and all $\ell\in\nat$, then
  \[
  \vmin(\theta) \leq v(\theta) \leq \vmax(\theta).
  \]
  But, combining the definition of $v(\theta)$
  with~\eqref{eq:period_depend_stab}, we immediately see that
  $v(\theta)$ satisfies the stability condition in
  Lemma~\ref{lem:stab_cond}, and thus $v(\theta)$ is a field stabilizing
  position-feedback controller.
  \end{proof}



\begin{thebibliography}{10}
\providecommand{\url}[1]{#1}
\csname url@samestyle\endcsname
\providecommand{\newblock}{\relax}
\providecommand{\bibinfo}[2]{#2}
\providecommand{\BIBentrySTDinterwordspacing}{\spaceskip=0pt\relax}
\providecommand{\BIBentryALTinterwordstretchfactor}{4}
\providecommand{\BIBentryALTinterwordspacing}{\spaceskip=\fontdimen2\font plus
\BIBentryALTinterwordstretchfactor\fontdimen3\font minus
  \fontdimen4\font\relax}
\providecommand{\BIBforeignlanguage}[2]{{%
\expandafter\ifx\csname l@#1\endcsname\relax
\typeout{** WARNING: IEEEtran.bst: No hyphenation pattern has been}%
\typeout{** loaded for the language `#1'. Using the pattern for}%
\typeout{** the default language instead.}%
\else
\language=\csname l@#1\endcsname
\fi
#2}}
\providecommand{\BIBdecl}{\relax}
\BIBdecl

\bibitem{KakalisVentikosJouranlOfHazardousMaterials08}
N.~M.~P. Kakalis and Y.~Ventikos, ``Robotic swarm concept for efficient oil
  spill confrontation,'' \emph{Journal of Hazardous Materials}, vol. 154, pp.
  880--887, 2008.

\bibitem{MacKenzieBalchAAAI93RobotVacuuming}
D.~MacKenzie and T.~Balch, ``Making a clean sweep: Behavior-based vacuuming,''
  in \emph{Proceedings of the AAAI Fall Symposium: Instantiating Real-world
  Agents}, Raleigh, NC, March 1993.

\bibitem{NC-NA.ea:09}
N.~Correll, N.~Arechiga, A.~Bolger, M.~Bollini, B.~Charrow, A.~Clayton,
  F.~Dominguez, K.~Donahue, S.~Dyar, L.~Johnson, H.~Liu, A.~Patrikalakis,
  T.~Robertson, J.~Smith, D.~Soltero, M.~Tanner, L.~White, and D.~Rus,
  ``Building a distributed robot garden,'' in \emph{IEEE/RSJ Int. Conf. on
  Intelligent Robots \& Systems}, St. Louis, MO, 2009, pp. 1509--1516.

\bibitem{SmithJFR10UnderwaterGliders}
R.~N. Smith, M.~Schwager, S.~L. Smith, B.~H. Jones, D.~Rus, and G.~S. Sukhatme,
  ``Persistent ocean monitoring with underwater gliders: Adapting
  spatiotemporal sampling resolution,'' \emph{Journal of Field Robotics}, 2010,
  {S}ubmitted.

\bibitem{DunbabinAusConfOnRobAndAuto04ReefRobot}
M.~Dunbabin, J.~Roberts, K.~Usher, and P.~Corke, ``A new robot for
  environmental monitoring on the {G}reat {B}arrier {R}eef,'' in
  \emph{Australasian Conference on Robotics and Automation}, Australian
  National University, Canberra, December 2004.

\bibitem{SrinivasanVSSN04AirborneTrafficSurveillance}
S.~Srinivasan, H.~Latchman, J.~Shea, T.~Wong, and J.~McNair, ``Airborne traffic
  surveillance systems: video surveillance of highway traffic,'' in
  \emph{International workshop on Video Surveillance \& Sensor Networks}, New
  York, NY, USA, October 2004, pp. 131--135.

\bibitem{SLS-DR:09a}
S.~L. Smith and D.~Rus, ``Multi-robot monitoring in dynamic environments with
  guaranteed currency of observations,'' in \emph{{IEEE} Conf. on Decision and
  Control}, Atlanta, GA, Dec. 2010, pp. 514--521.

\bibitem{KantZuckerIJRR86PathVelocityDecomposition}
K.~Kant and S.~W. Zucker, ``Toward efficient trajectory planning: The
  path-velocity decomposition,'' \emph{International Journal of Robotics
  Research}, vol.~5, no.~3, pp. 72--89, 1986.

\bibitem{DGL:84}
D.~G. Luenberger, \emph{Linear and Nonlinear Programming}, 2nd~ed.\hskip 1em
  plus 0.5em minus 0.4em\relax Addison-Wesley, 1984.

\bibitem{EWC:00}
E.~W. Cheney, \emph{Introduction to Approximation Theory}, 2nd~ed.\hskip 1em
  plus 0.5em minus 0.4em\relax {AMS} Chelsea Publishing, 2000.

\bibitem{EJC-MBW:08}
E.~J. Candes and M.~B. Wakin, ``An introduction to compressive sampling,''
  \emph{{IEEE} Signal Processing Magazine}, pp. 21--30, Mar. 2008.

\bibitem{RS-JJES:92}
R.~Sanner and J.~Slotine, ``Gaussian networks for direct adaptive control,''
  \emph{IEEE Transactions on Neural Networks}, vol.~3, no.~6, pp. 837--863,
  1992.

\bibitem{TP-SS:03}
T.~Poggio and S.~Smale, ``The mathematics of learning: Dealing with data,''
  \emph{Notices of the American Mathematical Society}, vol.~50, no.~5, pp.
  537--544, 2003.

\bibitem{Gandin63ObjectiveAnalysis}
L.~S. Gandin, \emph{Objective Analysis of Meteorological Fields}.\hskip 1em
  plus 0.5em minus 0.4em\relax Jerusalem: Israeli Program for Scientific
  Translations, 1966, (originally published in Russian in 1963, Gidrometeor,
  Leningrad).

\bibitem{CressieMathematicalGeology90Kriging}
N.~Cressie, ``The origins of kriging,'' \emph{Mathematical Geology}, vol.~22,
  no.~3, pp. 239--252, 1990.

\bibitem{GrocholskyPhDThesis02}
B.~Grocholsky, ``Information-theoretic control of multiple sensor platforms,''
  Ph.D. dissertation, University of Sydney, 2002.

\bibitem{Lynch08}
K.~M. Lynch, I.~B. Schwartz, P.~Yang, and R.~A. Freeman, ``Decentralized
  environmental modeling by mobile sensor networks,'' \emph{IEEE Transactions
  on Robotics}, vol.~24, no.~3, pp. 710--724, June 2008.

\bibitem{Yang08}
P.~Yang, R.~A. Freeman, and K.~M. Lynch, ``Mulit-agent coordination by
  decentralized estimation and control,'' \emph{{IEEE} Transactions on
  Automatic Control}, vol.~53, no.~11, pp. 248--2496, December 2008.

\bibitem{CortesTAC09KrigedKalmanFilter}
J.~Cort{\'e}s, ``Distributed {K}riged {K}alman filter for spatial estimation,''
  \emph{IEEE Transactions on Automatic Control}, vol.~54, no.~12, pp.
  2816--2827, 2009.

\bibitem{Zhang10}
F.~Zhang and N.~E. Leonard, ``Cooperative filters and control for cooperative
  exploration,'' \emph{IEEE Transactions on Automatic Control}, vol.~55, no.~3,
  pp. 650--663, 2010.

\bibitem{GrahamCortesTAC10VoronoiTrajectory}
R.~Graham and J.~Cort{\'e}s, ``Adaptive information collection by robotic
  sensor networks for spatial estimation,'' \emph{IEEE Transactions on
  Automatic Control}, 2010, {S}ubmitted.

\bibitem{LeNyCDC09PathOptGP}
J.~L. Ny and G.~J. Pappas, ``On trajectory optimization for active sensing in
  gaussian process models,'' in \emph{{IEEE} Conf. on Decision and Control and
  Chinese Control Conference}, Shanghai, China, December 2009, pp. 6282--6292.

\bibitem{BourgaultIROS02}
F.~Bourgault, A.~A. Makarenko, S.~B. Williams, B.~Grocholsky, and H.~F.
  Durrant-Whyte, ``Information based adaptive robotic exploration,'' in
  \emph{IEEE/RSJ Int. Conf. on Intelligent Robots \& Systems}, 2002, pp.
  540--545.

\bibitem{EMA-SPF-JSBM:00}
E.~M. Arkin, S.~P. Fekete, and J.~S.~B. Mitchell, ``Approximation algorithms
  for lawn mowing and milling,'' \emph{Computational Geometry: Theory and
  Applications}, vol.~17, no. 1-2, pp. 25--50, 2000.

\bibitem{IR-VLS-APN-HC:04}
I.~Rekleitis, V.~Lee-Shue, A.~P. New, and H.~Choset, ``Limited communication
  multi-robot team based coverage,'' in \emph{{IEEE} Int. Conf. on Robotics and
  Automation}, New Orleans, LA, Apr. 2004, pp. 3462--3468.

\bibitem{KoenigSzymanskiLiuAnnalsofMathAI01Ants}
S.~Koenig, B.~Szymanski, and Y.~Liu, ``Efficient and inefficient ant coverage
  methods,'' \emph{Annals of Mathematics and Artificial Intelligence}, vol.~31,
  pp. 41--76, 2001.

\bibitem{HC:01}
H.~Choset, ``Coverage for robotics \textendash{} {A} survey of recent
  results,'' \emph{Annals of Mathematics and Artificial Intelligence}, vol.~31,
  no. 1-4, pp. 113--126, 2001.

\bibitem{BB-JPH-JV:08}
B.~Bethke, J.~P. How, and J.~Vian, ``Group health management of {UAV} teams
  with applications to persistent surveillance,'' in \emph{{A}merican {C}ontrol
  {C}onference}, Seattle, WA, Jun. 2008, pp. 3145--3150.

\bibitem{BB-JR-JPH-MAV-JV:10}
B.~Bethke, J.~Redding, J.~P. How, M.~A. Vavrina, and J.~Vian, ``Agent
  capability in persistent mission planning using approximate dynamic
  programming,'' in \emph{{A}merican {C}ontrol {C}onference}, Baltimore, MD,
  Jun. 2010, pp. 1623--1628.

\bibitem{YC:04}
Y.~Chevaleyre, ``Theoretical analysis of the multi-agent patrolling problem,''
  in \emph{IEEE/WIC/ACM Int. Conf. Intelligent Agent Technology}, Beijing,
  China, Sep. 2004, pp. 302--308.

\bibitem{YE-NA-GAK:07}
Y.~Elmaliach, N.~Agmon, and G.~A. Kaminka, ``Multi-robot area patrol under
  frequency constraints,'' in \emph{{IEEE} Int. Conf. on Robotics and
  Automation}, Roma, Italy, Apr. 2007, pp. 385--390.

\bibitem{NN-IK:08}
N.~Nigram and I.~Kroo, ``Persistent surveillance using multiple unmannded air
  vehicles,'' in \emph{IEEE Aerospace Conf.}, Big Sky, MT, May 2008, pp. 1--14.

\bibitem{DBK-RWB-RSH:08}
D.~B. Kingston, R.~W. Beard, and R.~S. Holt, ``Decentralized perimeter
  surveillance using a team of {UAVs},'' \emph{IEEE Transactions on Robotics},
  vol.~24, no.~6, pp. 1394--1404, 2008.

\bibitem{PFH-DS-MWS:07}
P.~F. Hokayem, D.~Stipanovi{\'c}, and M.~W. Spong, ``On persistent coverage
  control,'' in \emph{{IEEE} Conf. on Decision and Control}, New Orleans, LA,
  Dec. 2007, pp. 6130--6135.

\bibitem{LK:75}
L.~Kleinrock, \emph{Queueing Systems. Volume I: Theory}.\hskip 1em plus 0.5em
  minus 0.4em\relax John Wiley, 1975.

\bibitem{DJS-GJvR:93a}
D.~J. Bertsimas and G.~J. {van~Ryzin}, ``Stochastic and dynamic vehicle routing
  in the {E}uclidean plane with multiple capacitated vehicles,''
  \emph{Operations Research}, vol.~41, no.~1, pp. 60--76, 1993.

\bibitem{SLS-MP-FB-EF:09a}
S.~L. Smith, M.~Pavone, F.~Bullo, and E.~Frazzoli, ``Dynamic vehicle routing
  with priority classes of stochastic demands,'' \emph{SIAM Journal on Control
  and Optimization}, vol.~48, no.~5, pp. 3224--3245, 2010.

\bibitem{MA-PS:06}
M.~Ahmadi and P.~Stone, ``A multi-robot system for continuous area sweeping
  tasks,'' in \emph{{IEEE} Int. Conf. on Robotics and Automation}, Orlando, FL,
  May 2006, pp. 1724--1729.

\bibitem{GO-ADL-MV:02}
G.~Oriolo, A.~D. Luca, and M.~Vendittelli, ``{WMR} control via dynamic feedback
  linearization: design, implementation, and experimental validation,''
  \emph{IEEE Transactions on Control Systems Technology}, vol.~10, no.~6, pp.
  835--852, 2002.

\end{thebibliography}
\end{document}